\documentclass{article}

\usepackage[T1]{fontenc}
\usepackage[utf8]{inputenc}
\usepackage{microtype}
\usepackage[preprint]{icml2026} 
\usepackage{graphicx}
\usepackage{subcaption}
\usepackage{booktabs}
\usepackage{multirow}

\usepackage{enumitem}
\usepackage{float}            
\usepackage[section]{placeins}
\usepackage{dblfloatfix}      

\setlist[itemize]{topsep=2pt, itemsep=1pt, parsep=0pt, partopsep=0pt}
\setlist[enumerate]{topsep=2pt, itemsep=1pt, parsep=0pt, partopsep=0pt}

\usepackage{icml2026}

\usepackage{algorithm}
\usepackage{algorithmic}

\usepackage{amsmath,amssymb,mathtools,amsthm}

\newtheoremstyle{tightplain}%
  {0pt}{0pt}
  {\itshape}{}
  {\bfseries}{.}{0.5em}{}

\newtheoremstyle{tightdefinition}%
  {0pt}{0pt}%
  {\normalfont}{}%
  {\bfseries}{.}{0.5em}{}

\newtheoremstyle{tightremark}%
  {0pt}{0pt}%
  {\normalfont}{}%
  {\itshape}{.}{0.5em}{}

\theoremstyle{tightplain}
\newtheorem{theorem}{Theorem}[section]
\newtheorem{proposition}[theorem]{Proposition}
\newtheorem{lemma}[theorem]{Lemma}
\newtheorem{corollary}[theorem]{Corollary}

\theoremstyle{tightdefinition}

\newtheorem{assumption}[theorem]{Assumption}

\theoremstyle{tightremark}
\newtheorem{remark}[theorem]{Remark}

\makeatletter
\renewcommand\paragraph{\@startsection{paragraph}{4}{\z@}%
  {2pt}%
  {-0.8em}%
  {\normalfont\normalsize\bfseries}}
\makeatother

\newcommand{\R}{\mathbb{R}}
\newcommand{\T}{\mathbb{T}}
\newcommand{\N}{\mathbb{N}}
\newcommand{\E}{\mathbb{E}}
\newcommand{\PP}{\mathbb{P}}
\newcommand{\Law}{\operatorname{Law}}

\newcommand{\Cov}{\operatorname{Cov}}

\newcommand{\W}{W}
\newcommand{\Wtwo}{\W_2}
\newcommand{\Wone}{\W_1}

\newcommand{\fint}{\mathop{\kern0pt\int}\nolimits\!\!\!\!\!\!-}

\usepackage[most]{tcolorbox}
\definecolor{theoBack}{RGB}{252,248,236}  
\definecolor{theoFrame}{RGB}{214,196,150} 

\tcbset{
  theobox/.style={
    enhanced,
    colback=theoBack,
    colframe=theoFrame,
    boxrule=0.5pt,
    arc=2.2mm,
    left=6pt,right=6pt,top=4pt,bottom=4pt,
    boxsep=0pt,
    before skip=3pt,
    after skip=3pt,
    before upper={\setlength{\parskip}{0pt}\setlength{\parindent}{0pt}},
  }
}

\tcolorboxenvironment{theorem}{theobox}
\tcolorboxenvironment{proposition}{theobox}
\tcolorboxenvironment{lemma}{theobox}
\tcolorboxenvironment{corollary}{theobox}
\tcolorboxenvironment{definition}{theobox}
\tcolorboxenvironment{assumption}{theobox}
\tcolorboxenvironment{remark}{theobox}

\usepackage{hyperref}
\hypersetup{
  colorlinks=true,
  linkcolor=blue,
  citecolor=blue,
  urlcolor=blue
}
\usepackage[capitalize,noabbrev]{cleveref}

\usepackage[textsize=tiny]{todonotes}

\icmltitlerunning{Memory-Conditioned Flow-Matching for Stable Autoregressive PDE Rollouts}

\begin{document}

\twocolumn[
\icmltitle{Memory-Conditioned Flow-Matching for Stable Autoregressive PDE Rollouts}

\begin{icmlauthorlist}
  \icmlauthor{Victor Armegioiu}{eth}
\end{icmlauthorlist}

\icmlaffiliation{eth}{Department of Mathematics, ETH Z\"urich, Z\"urich, Switzerland}
\icmlcorrespondingauthor{Victor Armegioiu}{victor.armegioiu@math.ethz.ch}

\vskip 0.3in
]

\printAffiliationsAndNotice{}

\begin{abstract}
Autoregressive generative PDE solvers can be accurate one step ahead yet drift over long rollouts, especially in coarse-to-fine
regimes where each step must regenerate unresolved fine scales. This is the regime of diffusion and flow-matching generators:
although their internal dynamics are Markovian, rollout stability is governed by per-step \emph{conditional law} errors. Using
the Mori--Zwanzig projection formalism, we show that eliminating unresolved variables yields an exact resolved evolution with a
Markov term, a memory term, and an orthogonal forcing, exposing a structural limitation of memoryless closures. Motivated by
this, we introduce memory-conditioned diffusion/flow-matching with a compact online state injected into denoising via latent
features. Via disintegration, memory induces a structured conditional tail prior for unresolved scales and reduces the transport
needed to populate missing frequencies. We prove Wasserstein stability of the resulting conditional kernel. We then derive
discrete Gr\"onwall rollout bounds that separate memory approximation from conditional generation error. Experiments on
compressible flows with shocks and multiscale mixing show improved accuracy and markedly more stable long-horizon rollouts, with
better fine-scale spectral and statistical fidelity.
\end{abstract}

\section{Introduction}
\label{sec:intro}

\noindent
Many physical systems are modeled by time-dependent partial differential equations (PDEs)  \cite{evans2022partial}.
In this work we consider PDEs of the abstract form
\begin{equation*}\label{eq:intro-pde}
\partial_t u(t) = \mathcal{F}(u(t)), \quad u(0)=u_0,
\end{equation*}
where $u(t)$ is a (possibly vector-valued) field on a spatial domain $D$ and $\mathcal{F}$ is a nonlinear differential operator.
Concrete examples include the incompressible Navier--Stokes and compressible Euler equations governing multiscale fluid flows.
A recurring difficulty is that solutions can be highly sensitive to initial/boundary data: rather than a single trajectory, one often
cares about the \emph{evolving law} of solutions induced by uncertain inputs, i.e.\ the pushforward measure $\mu_t = (S_t)_\#\mu_0$
under the PDE solution operator $S_t$.

\noindent
\textbf{Why many--query regimes are hard.}
Sampling $\mu_t$ with Monte Carlo ensembles requires repeated high-fidelity PDE solves, which becomes infeasible when the state is
high-dimensional and multiscale structure must be resolved.
This has driven intense interest in learned surrogates.
Deterministic operator-learning approaches (e.g.\ neural operators and neural PDE solvers) are effective in many regimes
\cite{li2020fourier,kovachki2023neural,pathak2022fourcastnet,brandstetter2022message,liu2024multi,takamoto2022pdebench},
but for chaotic/multiscale flows they can be poorly matched to \emph{distributional} objectives and long-horizon rollouts. 

\noindent
\textbf{Generative solvers and the rollout-stability gap.}
Conditional diffusion and related generative models have been proposed to represent $\mu_t$ directly by learning stochastic
transitions; in fluids, GenCFD is a representative large-scale conditional diffusion solver \cite{molinaro2024generative}.
Flow matching and its relatives replace reverse-time SDE sampling by a learned transport in an internal time $\tau\in[0,1]$
\cite{lipman2022flow,albergo2023stochastic,tong2023improving}.
Rectified flows in particular provide an ODE-based transport with substantially fewer integration steps at sampling time,
often with improved efficiency \cite{armegioiu2025rectified,wang2024rectified}.
However, PDE surrogates are deployed \emph{autoregressively} in physical time, and good one-step accuracy does not ensure stable long rollouts: small law defects can compound. DySLIM~\cite{schiff2024dyslim} instead matches invariant-measure statistics of the time-step operator, but existence/uniqueness of an invariant measure is not automatic in our setting (transient regimes, shocks/mixing, discretization- and resolution-dependent stationarity, and possible non-uniqueness). We therefore avoid invariance assumptions and analyze finite-horizon law error directly; see Remark~\ref{rem:invarian-measuers} and Appendix~\ref{app:mmd_vs_flow}. Corollary~\ref{cor:K:fm-to-invariant} shows our flow-matching-base measure-transport automatically produces bounds for the invariant measure loss of~\cite{schiff2024dyslim}.

\noindent
\textbf{Where memory enters: Mori--Zwanzig and missing-frequency generation.}
When only partial (resolved) information is observed, unresolved degrees of freedom influence future evolution.
The Mori-Zwanzig (MZ) formalism \cite{hijon2010mori} makes this precise: under a resolved/unresolved split, eliminating unresolved variables yields an
\emph{exact} resolved evolution with a Markov term, a \emph{history (memory)} term, and an orthogonal forcing
\cite{hijon2010mori}.
Thus, purely Markov conditional kernels are generically misspecified when the conditioning discards information.
Recent work shows concrete gains from incorporating memory mechanisms for time-dependent PDE modeling \cite{ruiz2024benefits}.

\noindent
\textbf{Our setting: generative PDE solvers as conditional Paley--Littlewood predictors.}
In coarse-to-fine regimes, we condition on low modes and must sample a missing high-frequency tail.
A memoryless generator starts from a generic tail prior and must do substantial internal-time transport to populate absent shells
(Lemma~\ref{lem:W2-reduction}; Appendix~\ref{app:rf-budget}).
A compact memory can encode recent cross-scale transfer statistics to shape a more structured conditional tail prior
(Appendix~\ref{app:projection-observables}), thereby reducing the conditional mismatch that drives the one-step defect $\epsilon_n$
(Appendix~\ref{app:conditional}).

\noindent
\textbf{Method: Memory-Conditioned Rectified Flow (MCRF).}
We propose a rectified-flow transition kernel augmented with an explicit latent memory state updated along the autoregressive
trajectory.
The memory is extracted from the backbone representation and injected back into the denoising/transport dynamics, forcing history
dependence to be expressed in the same feature space that drives conditional generation.
Our memory module is compatible with stable state-space sequence models designed to approximate multi-timescale convolutions
\cite{gu2021efficiently,smith2208simplified,orvieto2023resurrecting,gu2024mamba,ma2022mega}.

\noindent
\textbf{Related directions.}
Our focus is complementary to physics-integrated or hybrid latent-physics approaches \cite{lutter2019deep,akhtar2024physics,yue2024deltaphi}
and to learned simulators for turbulence \cite{stachenfeld2022learned}:
we study law-level rollout stability for \emph{generative} coarse-to-fine solvers, and isolate a mechanism by which memory reduces
the intrinsic difficulty of populating missing shells.

\paragraph{Contributions.}
\begin{itemize}
  \item \textbf{Model:} a memory-conditioned rectified-flow transition kernel for stable autoregressive PDE rollouts
  \cite{armegioiu2025rectified,lipman2022flow,wang2024rectified}.
  \item \textbf{Architecture:} an extract/inject memory pathway coupled to the backbone representation, leveraging state-space
  memory mechanisms for multi-timescale history \cite{gu2021efficiently,gu2024mamba}.
  \item \textbf{Theory:} a law-level error recursion and a decomposition of one-step defects into memory-approximation and
  conditional-generation terms, together with a frequency-cutoff disintegration explaining how memory reduces tail mismatch
  \cite{ruiz2024benefits,hijon2010mori}.
\end{itemize}

\paragraph{Organization.}
Section~\ref{sec:background} fixes notation (PDE setting, MZ viewpoint, conditional flow matching, and the Littlewood--Paley cutoff
regime). Section~\ref{sec:theory} presents the rollout-stability mechanism and shows how memory shapes the tail prior and reduces
one-step defects. Section~\ref{sec:method} describes MCRF, and Section~\ref{sec:results} reports experimental instantiations of our method.

\section{Background and setup}
\label{sec:background}

We study autoregressive surrogate modeling for time-evolving PDEs under partial (resolved) observations.
The key object is the \emph{rollout law}: the model is iterated in physical time, so small one-step law defects can compound.
We fix notation for (i) the PDE and resolved operator, (ii) Mori--Zwanzig’s memory/noise decomposition, (iii) conditional
flow-matching / rectified-flow transitions, and (iv) the Littlewood--Paley cutoff viewpoint used in the theory.

\subsection{PDE setting, resolved observations, and distributional rollouts}
\label{subsec:bg-pde}

Let $D$ be either a flat torus $\T^d$ or a bounded $C^1$ domain.
Let $\mathcal{H}:=L^2(D;\R^{d_u})$ be the (fine) state space and consider
\begin{equation*}\label{eq:bg-pde-abstract2}
\partial_t u(t) = \mathcal{F}(u(t)),\quad u(0)=u_0\in\mathcal{H},
\end{equation*}
with (when relevant) boundary conditions appropriate to the PDE.
We denote by $S_t:\mathcal{H}\to\mathcal{H}$ the (possibly weak) solution operator, so $u(t)=S_t u_0$ when well-defined.

\textbf{Statistical solutions / evolving laws.}
Given an initial law $\mu_0\in\mathcal{P}_2(\mathcal{H})$, the evolving law is $\mu_t := (S_t)_\#\mu_0$.

\textbf{Resolved operator and autoregressive model.}
Fix a bounded linear observation operator $\Pi_R:\mathcal{H}\to\mathcal{Y}$ into a resolved space $\mathcal{Y}$ and define
$y(t):=\Pi_R u(t)$.
We discretize physical time with step $\Delta t>0$ and $t_n=n\Delta t$.
A learned \emph{distributional} solver produces a Markov chain $(\widehat u_n)_{n\ge 0}$ via
\begin{equation*}\label{eq:bg-rollout2}
\widehat u_{n+1}\sim \widehat K_{\theta,n}(\,\cdot\,\mid y_n,\widehat m_n,t_n),\quad
y_n := \Pi_R \widehat u_n,
\end{equation*}
where $\widehat m_n\in\R^{d_m}$ is a learned memory state updated causally from the trajectory (Section~\ref{sec:method}).
We compare the rollout law $\widehat\mu_n:=\Law(\widehat u_n)$ to the true law $\mu_n:=\Law(u(t_n))$.

\subsection{Mori--Zwanzig: exact memory and orthogonal forcing after elimination}
\label{subsec:bg-mz2}

Mori--Zwanzig (MZ) formalizes the effect of eliminating unresolved variables.
For exposition, consider a finite-dimensional discretization $H$ and induced ODE $\dot u = F(u)$ with flow $S_t$.
Let $\mathcal{A}$ be a class of smooth observables $\psi:H\to\R^k$ and define the Liouville operator
$(L\psi)(u):=D\psi(u)\,F(u)$ so that $(e^{tL}\psi)(u_0)=\psi(S_tu_0)$.
Fix a projection $P:\mathcal{A}\to\mathcal{A}$ corresponding to ``resolved information'' and set $Q:=I-P$.
Then for any $\psi\in\mathcal{A}$ one has the identity
\begin{equation*}\label{eq:bg-mz2}
\begin{aligned}
\frac{d}{dt}\,P e^{tL}\psi
&= P e^{tL}PL\psi
+\int_0^t P e^{(t-s)L}\,PL\,e^{sQL}\,QL\psi\,ds \\
&\quad + P e^{tQL}\,QL\psi.
\end{aligned}
\end{equation*}

The first term is Markovian, the integral term is explicit \emph{history dependence} (memory), and the last term is an
\emph{orthogonal forcing} generated by unresolved dynamics \cite{hijon2010mori}.
Unless the latter two vanish, any closure depending only on the instantaneous resolved state is intrinsically misspecified.
This motivates augmenting \eqref{eq:bg-rollout2} with a memory state, consistent with empirical findings in \cite{ruiz2024benefits}.

\subsection{Conditional flow matching / rectified flows as stochastic transitions}
\label{subsec:bg-fm2}

We implement $\widehat K_{\theta,n}$ by conditional transport in an internal time $\tau\in[0,1]$.
Let $c:=(y,m,t)$ denote conditioning.
Given a base law $\nu_0(\cdot\mid c)\in\mathcal{P}_2(\mathcal{H})$ and a learned vector field
$v_\theta:[0,1]\times\mathcal{H}\times\mathcal{C}\to\mathcal{H}$, define
\begin{equation*}\label{eq:bg-rf-ode2}
\frac{d}{d\tau}U_\tau = v_\theta(\tau,U_\tau;c),
\quad U_0\sim \nu_0(\cdot\mid c),
\end{equation*}
and set $\widehat u := U_1$.
This induces a conditional kernel
\begin{equation*}\label{eq:bg-rf-kernel2}
\widehat K_{\theta,n}(\cdot\mid c) := \Law(\widehat u\mid c) = (\Phi^c_\theta)_\#\,\nu_0(\cdot\mid c),
\end{equation*}
where $\Phi^c_\theta$ is the time-$1$ flow map of \eqref{eq:bg-rf-ode2}.
This includes flow matching and related stochastic-interpolant viewpoints \cite{lipman2022flow,albergo2023stochastic,tong2023improving},
and rectified-flow variants that emphasize efficient ODE sampling \cite{armegioiu2025rectified,wang2024rectified}.

\paragraph{Why internal time matters for PDE rollouts.}
In coarse-to-fine settings, part of the state is effectively \emph{prior-driven} at intermediate $\tau$ (low SNR), so the choice of
conditional prior--and thus the conditioning variables--directly affects one-step defects and rollout stability
(Section~\ref{sec:theory}).

\subsection{Littlewood--Paley cutoff and the ``transport difficulty'' viewpoint}
\label{subsec:bg-resolution2}

We now specialize to the spectral-cutoff regime used in the theory.
For clarity take $D=\T^d$ and $\mathcal{H}=L^2(\T^d;\R^{d_u})$.

\textbf{Dyadic projectors and resolved/unresolved split.}
Let $(\Delta_j)_{j\ge -1}$ be a Littlewood--Paley decomposition on $\T^d$ and fix a cutoff $J\in\N$.
Define
\begin{equation*}\label{eq:bg-lp-proj2}
S_J u := \sum_{j\le J}\Delta_j u,\quad
Q_J u := u - S_J u = \sum_{j>J}\Delta_j u.
\end{equation*}
We take $\Pi_R := S_J$ so that each state decomposes as
\[
u = y + z,\quad y:=S_Ju\in\mathcal{H}_{\le J},\quad z:=Q_Ju\in\mathcal{H}_{>J}.
\]
In coarse-to-fine generation we condition on $y$ and must sample the tail $z$.

\paragraph{Why memory changes the difficulty of populating missing shells.}
A memoryless generator typically uses a generic conditional tail prior (e.g.\ isotropic noise on $\mathcal{H}_{>J}$).
Let $\pi_J(\cdot\mid y)$ denote such a base tail prior, and write the true conditional tail law as $\Law(z\mid y)$.
The mismatch
\[
\delta_J(y):=W_2\!\big(\Law(z\mid y),\,\pi_J(\cdot\mid y)\big)
\]
lower bounds how much transport the internal-time dynamics must perform before it can match the true tail.
A memory state $m$ refines the conditioning, replacing $\Law(z\mid y)$ by $\Law(z\mid y,m)$; in general, conditioning reduces
best-achievable prediction error for tail summaries and shrinks uncertainty.
This is precisely the mechanism formalized in Section~\ref{sec:theory}:
(i) disintegration localizes $W_2$ mismatch to $\mathcal{H}_{>J}$ when low modes are pinned,
(ii) conditioning on $m$ reduces optimal $L^2$ errors for tail statistics, and
(iii) if the true conditional kernel depends on $m$, then memoryless kernels incur an irreducible conditional gap.
These effects translate into smaller one-step law defects and hence more stable rollouts.

\section{Theory: memory-shaped priors and stable autoregressive rollouts}
\label{sec:theory}

We analyze diffusion/flow-matching generators used \emph{autoregressively} as PDE solvers.
In coarse-to-fine regimes, each step conditions on resolved content and must regenerate unresolved structure (typically high
frequencies). Small per-step law defects can therefore compound and destabilize long rollouts.

Our theory isolates one mechanism and propagates it to a rollout bound:
(i) flow matching / rectified flows evolve samples in an \emph{internal time} $\tau$, where high-frequency shells can be
\emph{prior-driven} (low SNR); (ii) Mori--Zwanzig predicts that the correct tail prior is \emph{history dependent};
(iii) a finite-dimensional memory can approximate the relevant history convolution, shaping a structured tail prior and reducing
an intrinsic Wasserstein mismatch that lower bounds the transport needed to populate missing shells.

\textbf{Proofs.} All statements in this section are proved in Appendix~\ref{app:theory}, under the (tacit) assumptions justified in Appendix~\ref{app:regularity-assumptions}.

\subsection{Internal time and Mori--Zwanzig: prior-driven shells require a history convolution}
\label{subsec:theory-mechanism}

We work in the spectral-cutoff setting of Section~\ref{subsec:bg-resolution2}.
Fix a Littlewood--Paley decomposition $(\Delta_j)_{j\ge -1}$ on $\T^d$ and a cutoff $J$.
At physical step $n$, decompose the fine state as
\[
u_n = y_n + z_n,
\quad
y_n := S_J u_n \in \mathcal{H}_{\le J},
\quad
z_n := Q_J u_n \in \mathcal{H}_{>J}.
\]
The unresolved tail $z_n$ influences the next state; in a Mori--Zwanzig reduction, this influence appears through memory plus an
orthogonal forcing term (Section~\ref{subsec:bg-mz2}).

\textbf{A minimal linear-Gaussian tail surrogate.}
Let $Z_n := z_n \in \mathcal{H}_{>J}$ and let $E_n \in \R^{d_E}$ be a resolved embedding extracted from $y_n$
(e.g.\ a UNet bottleneck feature), $E_n := \mathrm{Enc}(y_n)$.
Assume a stable linear state-space update
\begin{equation*}\label{eq:theory-tail-dynamics}
Z_{n+1} = A Z_n + B E_n + \eta_n,
\end{equation*}
where $A:\mathcal{H}_{>J}\to\mathcal{H}_{>J}$ has spectral radius $<1$, $B:\R^{d_E}\to\mathcal{H}_{>J}$ is bounded, and
$(\eta_n)_n$ is i.i.d.\ centered Gaussian in $\mathcal{H}_{>J}$.

\textbf{Internal time as corruption level.}
In diffusion/flow-matching, intermediate internal times correspond to Gaussian corruption of what is being generated.
For $\tau\in[0,1]$, model this by
\[
X^{(\tau)}_{n+1} := Z_{n+1} + \sigma(\tau)\,\xi,
\quad
\xi \sim \mathcal{N}(0,I)
\ \text{independent}.
\]
When $\sigma(\tau)$ dominates the signal in shell $j$, that shell is low SNR and denoising must fall back on a conditional prior. \\

\begin{theorem}[Shell-wise history prior in the optimal denoiser]
\label{thm:tau-shell-memory-fixed}
Let $\mathcal{F}_n := \sigma(E_0,\dots,E_n)$. Assume that, conditionally on $\mathcal{F}_n$, each shell covariance is isotropic:
\[
\Cov\!\left(\Delta_j Z_{n+1}\,\big|\,\mathcal{F}_n\right)=s_{j,n+1}^2\,I_j,
\quad j>J.
\]
Define the history term (conditional mean)
\[
M^\star_{n+1}
:= \E\!\left[ Z_{n+1}\,\big|\,\mathcal{F}_n \right]
= A^{n+1}Z_0 \;+\; \sum_{k=0}^{n} A^{\,n-k} B E_k .
\]
Fix $\tau\in[0,1]$ and set $X:=X^{(\tau)}_{n+1}$. Then for every $j>J$,
\[
\begin{aligned}
\E\!\left[\Delta_j Z_{n+1}\,\big|\,X,\mathcal{F}_n\right]
&= \kappa_{j,n+1}(\tau)\,\Delta_j X \\
&\quad + \bigl(1-\kappa_{j,n+1}(\tau)\bigr)\,\Delta_j M^\star_{n+1},
\end{aligned}
\]
where
\[
\kappa_{j,n+1}(\tau)
:= \left(1+\frac{\sigma(\tau)^2}{s_{j,n+1}^2}\right)^{-1}\in[0,1].
\]
In particular, in low-SNR shells (i.e.\ $s_{j,n+1}^2\ll \sigma(\tau)^2$), the denoiser is prior-driven and governed by
$\Delta_j M^\star_{n+1}$.
\end{theorem}

\paragraph{Why this forces memory.}
The term $M^\star_{n+1}$ is a causal convolution of $(E_k)_{k\le n}$ with weights $A^{n-k}$.
Thus, precisely in the shells where internal-time denoising cannot ``see'' the signal, generation must rely on a history-shaped
prior. A learned memory state is a natural approximation class for this convolution, and shaping the tail prior \emph{before}
transport reduces the mismatch that a memoryless generator must correct by moving mass across missing shells.

\subsection{Rollouts: stability reduces everything to a one-step law defect}
\label{subsec:theory-rollout-fixed}

Let $\mathcal{H}:=L^2(D;\R^{d_u})$ and let $S_{\Delta t_n}:\mathcal{H}\to\mathcal{H}$ be the PDE solution operator over the
$n$-th physical step. For laws $\mu\in\mathcal{P}_2(\mathcal{H})$, the true evolution satisfies
$\mu_{n+1}=(S_{\Delta t_n})_\#\mu_n$.
A learned autoregressive model produces rollout laws $(\widehat\mu_n)_{n\ge 0}$ and we track
$d_n := W_2(\mu_n,\widehat\mu_n)$. Define the one-step law defect
\vspace{-0.25cm}
\begin{equation}\label{eq:theory-epsn-fixed}
\epsilon_n
:= W_2\!\bigl((S_{\Delta t_n})_\#\widehat\mu_n,\ \widehat\mu_{n+1}\bigr).
\end{equation}

\begin{lemma}[Error recursion]\label{lem:recursion-fixed}
Under Assumption~\ref{ass:phys-stab-fixed},
\[
d_{n+1} \le e^{\alpha_n} d_n + \epsilon_n .
\]
\end{lemma}
\begin{remark}[Relation to invariant-measure regularization] \label{rem:invarian-measuers}
Invariant-measure objectives (e.g.\ MMD between stationary laws) target a fixed point of the time-step operator.
By contrast, our analysis controls the operator itself via one-step conditional law errors.
In particular, if the true coarse-grained Markov operator is contractive/mixing in a weak metric (e.g.\ an IPM/MMD),
then the stationary-law discrepancy is bounded by the one-step kernel discrepancy; see Appendix~\ref{app:mmd_vs_flow}.
\end{remark}

\subsection{Disintegration: with low modes fixed, mismatch localizes to the tail}
\label{subsec:theory-disintegration-fixed}

Fix a cutoff $J$ and write $u=y+z$ with $y=S_Ju\in\mathcal{H}_{\le J}$ and $z=Q_Ju\in\mathcal{H}_{>J}$.
Disintegrating a law $\mu\in\mathcal{P}_2(\mathcal{H})$ along $y$ yields conditionals $\mu(\cdot\mid y)$ and the decomposition
$\mu(du)=\int \mu(du\mid y)\,\mu_{\le J}(dy)$, where $\mu_{\le J}=(S_J)_\#\mu$.

\begin{lemma}[$W_2$ reduction under pinned low modes]\label{lem:W2-reduction-fixed}
Fix $y\in\mathcal{H}_{\le J}$ and let $\rho,\tilde\rho\in\mathcal{P}_2(\mathcal{H}_{>J})$.
If $\mu=\delta_y\otimes\rho$ and $\tilde\mu=\delta_y\otimes\tilde\rho$, then
\[
W_2(\mu,\tilde\mu)=W_2(\rho,\tilde\rho),
\]
with the right-hand $W_2$ computed in $\mathcal{H}_{>J}$.
\end{lemma}

This formalizes: once resolved content is fixed, all remaining uncertainty (and mismatch) lies in the missing-frequency tail.

\subsection{Memory reduces tail mismatch}
\label{subsec:theory-memory-helps}

Mori--Zwanzig predicts that unresolved influence is history dependent, so conditioning only on $Y$ is generically insufficient.
Two complementary statements capture why memory helps: it reduces the best achievable tail prediction error, and it prevents an
irreducible conditional-kernel gap when the truth varies with history.

\begin{lemma}[Conditioning reduces $L^2$ projection error]\label{lem:cond-proj-fixed}
Let $\mathcal{G}$ be a Hilbert space and $G(Z)\in L^2(\Omega;\mathcal{G})$.
Set $\Pi_Y:=\E[\cdot\mid Y]$ and $\Pi_{Y,M}:=\E[\cdot\mid Y,M]$.
If $\sigma(Y)\subset\sigma(Y,M)$, then
\[
\|G(Z)-\Pi_{Y,M}G(Z)\|_{L^2(\Omega)}
\le
\|G(Z)-\Pi_YG(Z)\|_{L^2(\Omega)}.
\]
\end{lemma}

\paragraph{Interpretation.}
If $G(Z)$ encodes shell energies, flux proxies, or subgrid-stress summaries, then conditioning on a memory $M$ can only improve
their best $L^2$ prediction. Through \Cref{thm:tau-shell-memory-fixed}, these are exactly the quantities that internal-time
denoising must supply in low-SNR shells; hence memory can reduce the tail mismatch \emph{before} transport.

\begin{proposition}[Irreducible mismatch without memory]\label{prop:irreducible-fixed}
Fix $y$ and suppose there exist $m_1\neq m_2$ such that
$K_n^\star(\cdot\mid y,m_1)\neq K_n^\star(\cdot\mid y,m_2)$.
Let
\[
\delta_y := W_2\!\left(K_n^\star(\cdot\mid y,m_1),\ K_n^\star(\cdot\mid y,m_2)\right).
\]
Then for any memoryless kernel $K(\cdot\mid y)$,
\[
\max_{i\in\{1,2\}}\
W_2\!\left(K(\cdot\mid y),\ K_n^\star(\cdot\mid y,m_i)\right)
\ge \tfrac12\,\delta_y .
\]
\end{proposition}
This shows memory is not merely beneficial: if the true conditional law varies with history, a memoryless model has an intrinsic
one-step defect lower bound.

\subsection{From conditional mismatch to step defect, and a rollout bound}
\label{subsec:theory-rollout-bound-fixed}

Let $\widehat K_{\theta,n}(\cdot\mid y,m)$ denote the learned conditional transition kernel and $K_n^\star(\cdot\mid y,m)$ the true
conditional law. We separate
(i) \emph{memory approximation} (using $\widehat m_n$ instead of an ideal $m_n^\star$) and
(ii) \emph{conditional generation} at fixed memory.


\begin{proposition}[Conditional defect decomposition]\label{prop:cond-defect-fixed}
For fixed $n$ and $y$, define
\[
\widehat K_n^{\,y}(m) := \widehat K_{\theta,n}(\cdot\mid y,m),
\quad
K_n^{\star,y}(m) := K_n^\star(\cdot\mid y,m),
\]
and
\[
\varepsilon_{\mathrm{gen},n}(y,m)
:= W_2\!\bigl(\widehat K_n^{\,y}(m),\ K_n^{\star,y}(m)\bigr).
\]
Then, under Assumption~\ref{ass:lip-mem-fixed}, it holds that
\[
\begin{aligned}
W_2\!\bigl(\widehat K_n^{\,y}(\widehat m_n),\ K_n^{\star,y}(m_n^\star)\bigr)
&\le L_{\mathrm{mem}}\|\widehat m_n-m_n^\star\| \\
&\quad +\ \varepsilon_{\mathrm{gen},n}(y,m_n^\star).
\end{aligned}
\]
\end{proposition}

To pass from conditional mismatch to the law-level defect $\epsilon_n$, we use a standard mixture-coupling estimate that bounds
$W_2$ between mixtures by the root-mean-square of conditional $W_2$ distances (stated and proved in Appendix~\ref{app:theory}).

\begin{theorem}[Rollout bound]\label{thm:rollout-fixed}
Assume \Cref{ass:phys-stab-fixed} and the mixture-coupling estimate of Appendix~\ref{app:theory}.
Let $\beta_{k,n}:=\sum_{j=k}^{n-1}\alpha_j$ with $\beta_{n,n}=0$. Then
\[
d_n
\le
e^{\beta_{0,n}}\,d_0
+\sum_{k=0}^{n-1} e^{\beta_{k+1,n}}\,\epsilon_k,
\]
and each $\epsilon_k$ satisfies
\[
\begin{aligned}
\epsilon_k
&\le
L_{\mathrm{mem}}
\Bigl(\E\|\widehat m_k-m_k^\star\|^2\Bigr)^{1/2} \\
&\quad +\ 
\Bigl(\E\,\varepsilon_{\mathrm{gen},k}(Y_k,M_k^\star)^2\Bigr)^{1/2},
\end{aligned}
\]
where $(Y_k,M_k^\star)$ denotes the step-$k$ conditioning pair under the model rollout.
\end{theorem}

\paragraph{Summary.}
Internal time creates prior-driven shells, and \Cref{thm:tau-shell-memory-fixed} identifies the relevant prior as a shell-wise
history convolution.
\Cref{lem:W2-reduction-fixed} localizes mismatch to unresolved modes, \Cref{lem:cond-proj-fixed} shows memory can reduce tail
uncertainty, and \Cref{prop:irreducible-fixed} exhibits an intrinsic gap for memoryless kernels when the truth varies with
history. These mechanisms enter rollout stability through the one-step defect \eqref{eq:theory-epsn-fixed} and the bound of
\Cref{thm:rollout-fixed}. For a theoretical short illustration (incompressible NSE) of the subgrid effects induced by missing frequencies see Appendix~\ref{app:subgrid}, in particular Proposition~\ref{prop:app-RJ-tail}.


\section{Method and experimental protocol: memory-conditioned flow matching for autoregressive PDE rollouts}
\label{sec:method}

We learn \emph{stochastic} next-step transition laws for time-dependent PDEs and deploy them \emph{autoregressively} as surrogate
solvers. At test time, sampled predictions are fed back for 10--20 steps, and we impose \emph{structured lead-time schedules}
(uniform, sparse, increasing $\Delta t$) that are \emph{not} seen as such during training. This explicitly targets an
out-of-distribution (OOD) deployment regime in which (i) the input distribution drifts under feedback and (ii) the temporal
pattern of lead times differs from training.

We focus on coarse-to-fine conditioning in compressible Euler: each step must regenerate unresolved content from a degraded
current state. To represent the history dependence induced by cross-scale interactions (Mori--Zwanzig viewpoint,
Section~\ref{sec:background}), we maintain a compact memory state updated online and inject it into the generative backbone at the
bottleneck (see Appendix~\ref{app:memory-algorithms}). The generator itself remains Markovian in the \emph{internal} flow-matching time, while the overall surrogate is
non-Markovian in \emph{physical} time through memory.

\subsection{Task: autoregressive time stepping as conditional generation}
\label{subsec:method-task}

Let $u_n \in \mathcal{X}$ denote the PDE state at physical time $t_n$. In our benchmarks,
$\mathcal{X}=L^2(\T^2;\R^4)$ and $u=(\rho,\rho v_x,\rho v_y,E)$ consists of density, two momentum components, and total energy.

A learned solver is deployed autoregressively: given $(u_n,\Delta t_n)$ it samples a next state $\widehat u_{n+1}$ and feeds it
back as input at the subsequent step. Even if one-step prediction is accurate under ground-truth conditioning, long rollouts can
drift because the conditioning distribution shifts: repeated sampling progressively degrades fine-scale content, so later steps
effectively become coarse-to-fine reconstructions from self-generated inputs. We view each step as learning a conditional next-state law. Let $\mathcal{I}_n$ denote the information used at step $n$. The
target object is the true conditional kernel
\[
K^\star(\cdot\mid \mathcal{I}_n)\;\equiv\;\Law(u_{n+1}\mid \mathcal{I}_n).
\]
Motivated by Mori--Zwanzig, $K^\star(\cdot\mid \mathcal{I}_n)$ is generically \emph{not} Markov in instantaneous resolved
information alone once unresolved scales are eliminated; we therefore augment conditioning with a learned memory state.

\subsection{Memory state and bottleneck injection}
\label{subsec:method-memory}

We maintain a per-trajectory memory $m_n\in\R^{d_m}$ updated online from a compact embedding of the current step:
\[
e_n=\mathsf{E}_\psi(u_n,\Delta t_n),\quad
m_{n+1}=\mathsf{M}_\psi(m_n,e_n).
\]
Here $\mathsf{E}_\psi$ is an encoder producing a low-dimensional step summary, and $\mathsf{M}_\psi$ is a stable state-space
recurrence (SSM/S4 in our implementation) designed to capture multi-timescale history.

\textbf{Bottleneck fusion.}
Let $h_n$ denote the UNet bottleneck tensor produced when processing the (possibly degraded) current state and conditioning
signals. We inject memory at the bottleneck via a gated fusion
$
\widetilde h_n=\mathsf{Fuse}(h_n,m_n),
$
and decode $\widetilde h_n$ to drive the denoising/transport dynamics. This forces history dependence to live in the same latent
representation space that controls generation, rather than being appended as a shallow conditioning token.

\subsection{Conditional flow matching step (rectified-flow instantiation)}
\label{subsec:method-generator}

Our framework is compatible with diffusion/score models and flow-matching variants; in experiments we instantiate it with a
rectified-flow-style flow matching objective.

\textbf{Conditional kernel.}
Let $c_n:=(u_n,m_n,\Delta t_n)$ denote the conditioning at step $n$. The generator defines a conditional transition kernel
$
\widehat K_\theta(\cdot\mid c_n)\;:=\;\Law(\widehat u_{n+1}\mid c_n).
$

\textbf{Internal time $\tau$ (corruption/SNR level).}
Flow matching evolves samples in an \emph{internal time} $\tau\in[0,1]$ controlling the interpolation/corruption level between a
simple base draw and the target. Importantly, $\tau$ is \emph{not} a frequency index; spectral notions are handled separately
(e.g.\ via Littlewood--Paley shells in Section~\ref{subsec:bg-resolution2} and in the theory of Section~\ref{sec:theory}).

\textbf{Flow matching loss.}
Given a supervised pair $(u_n,u_{n+1})$ with conditioning $c_n$, sample a base draw
$u_0\sim\nu_0(\cdot\mid c_n)$ and $\tau\sim\mathrm{Unif}[0,1]$. Form the interpolation
$u_\tau=(1-\tau)u_0+\tau u_{n+1}$ and regress the velocity field $v_\theta$:
\begin{equation*}\label{eq:method-fm-loss}
\mathcal{L}_{\mathrm{FM}}(\theta,\psi)
=
\E\Big[\big\|v_\theta(\tau,u_\tau; c_n) - (u_{n+1}-u_0)\big\|_{L^2_x}^2\Big].
\end{equation*}
\vspace{-0.05cm}
At inference, we sample $u_0\sim\nu_0(\cdot\mid c_n)$ and integrate the induced ODE in $\tau$ from $0$ to $1$ to obtain
$\widehat u_{n+1}$.

All architectural and optimization details (UNet widths, SSM parameters, optimizer and schedule) are reported in
Appendix~\ref{app:impl}.

\subsection{Training protocol: one-step pairs + short randomized windows}
\label{subsec:method-training}

Training combines broad one-step coverage with limited rollout-awareness.

\textbf{Single-step pairs (variable lead times).}
We sample transitions $(u_n,u_{n+1})$ with lead times $\Delta t_n\in[0.05,0.30]$ to cover a range of horizons. This provides
one-step supervision across $\Delta t$ and enables evaluation binned by lead time.

\textbf{Short unrolls (teacher forcing on randomized windows).}
To partially align with autoregressive deployment while keeping training stable, we also train on randomized trajectory windows
of length at most 5 steps. Concretely, we unroll the model for 3--5 steps, update $m_n$ online, and apply teacher forcing with
probability $p_{\mathrm{TF}}$ (reported in Appendix~\ref{app:impl}); otherwise we feed back sampled predictions. This creates a
controlled form of distribution shift during training while remaining far shorter than the 10--20 step rollouts used at test
time (see Appendix~\ref{app:impl-training-detailed}).

\textbf{Downsampled-input degradation.}
To mimic the progressive loss of fine scales under feedback, we optionally degrade the conditioning input $u_n$ by
$2\times 2$ average pooling followed by bilinear upsampling, while keeping targets $u_{n+1}$ at full resolution. This directly
targets robustness to autoregressive degradation (details and ablations in Appendix~\ref{app:exp}, ~\ref{app:exp-baselines}).

\subsection{Benchmarks, deployment schedules, and metrics}
\label{subsec:method-benchmarks-eval}

\textbf{Benchmarks.}
We evaluate on two families of 2D compressible Euler trajectories exhibiting shocks and multiscale structure:
(i) Richtmyer--Meshkov-type shock--interface interaction (CE-RM), and
(ii) four-quadrant Riemann problems with interacting waves (CE-RP/CRP2D).
We use datasets from the PDEgym collection on the Hugging Face Hub; the dataset cards document the numerical setups and train/val/test splits \citep{pdegym_hf}. See also \citep{herde2024poseidon} (Appendix~\textbf{B}).

\textbf{Explicit OOD deployment: long rollouts and structured lead-time patterns.}
Training uses one-step pairs and short randomized windows (up to 5 steps). In contrast, evaluation rolls out autoregressively for
10--20 steps, feeding back sampled predictions throughout, and uses \emph{structured} lead-time schedules such as:
(i) uniform $\Delta t=0.05$,
(ii) sparse $\Delta t=0.10$,
and (iii) increasing schedules (e.g.\ $\Delta t_i=0.05+0.025\,i$).
These schedules stress temporal robustness under deployment shift: the model must remain stable when the sequence pattern of
lead times is fixed and long-range, unlike the randomized short contexts seen in training. See Appendix~\ref{app:impl-training-detailed} for precise algorithmic details.

\textbf{Metrics.}
We report (i) one-step relative $L^2$ errors (binned by $\Delta t$), and (ii) final-time relative $L^2$ errors after 10--20 step
autoregressive rollouts under the above schedules. To localize multiscale behavior, we additionally report Fourier-band /
spectral diagnostics (see also Appendix~\ref{app:exp}).

\section{Results: memory stabilizes long-horizon rollouts and preserves multi-scale structure}
\label{sec:results}

We compare memory-conditioned flow matching to a capacity-matched memoryless baseline on two 2D compressible Euler benchmarks
(CE-RM and CRP2D). The predicted state contains the four standard channels: density $\rho$, velocities $(u_x,u_y)$, and total
energy $E$. All models are deployed \emph{autoregressively}: sampled predictions are fed back as inputs for 10--20 steps.
The \emph{memory benefit} reported in tables is the relative reduction in final-time relative $\ell_2$ error (higher is better).

\textbf{Deliberate rollout-OOD evaluation.}
Training uses (i) one-step supervision and (ii) teacher forcing on \emph{randomized} windows of length at most $5$ steps.
Evaluation is strictly harder: we run 10--20-step autoregressive rollouts far deeper into self-generated input drift and impose
\emph{structured} lead-time schedules that do not occur as such during training (dense $\Delta t=0.05$, sparse $\Delta t=0.10$,
increasing $\Delta t_i=0.05+0.025\,i$, and variable-$\Delta t$ regimes).

\textbf{Single-step sanity check (supervised regime).}
Before stressing long autoregressive deployment, we verify that both models are already accurate for one-step prediction on
ground-truth inputs across lead times. \Cref{tab:single-step-baseline} shows low relative $\ell_2$ errors for short lead times
(e.g.\ $\Delta t\le 0.15$) and negligible differences between memory-conditioned and memoryless models (within $\pm 2\%$).
Thus, the gains reported below are not explained by better one-step accuracy, but by reduced compounding under rollout-OOD
distribution drift and lead-time pattern shift.
\begin{table}[t]
\centering
\caption{%
\textbf{Single-step baseline performance.}
Relative $\ell_2$ errors (mean $\pm$ std) for one-step prediction at short lead times.
Memory has negligible effect in this supervised regime, isolating its benefit to long-horizon rollout stability.}
\label{tab:single-step-baseline}

\scriptsize
\setlength{\tabcolsep}{3pt}
\renewcommand{\arraystretch}{0.95}

\begin{tabular}{@{}lccc@{}}
\toprule
\textbf{Dataset / Training} & $\Delta t$ & \textbf{w/ Mem} & \textbf{w/o Mem} \\
\midrule
\textit{CE-RM / Regular} & 0.05 & $0.079 \pm 0.023$ & $0.078 \pm 0.023$ \\
                        & 0.10 & $0.116 \pm 0.030$ & $0.116 \pm 0.031$ \\
                        & 0.15 & $0.145 \pm 0.044$ & $0.146 \pm 0.040$ \\
\addlinespace[0.25em]
\textit{CE-RM / Downsampled} & 0.05 & $0.130 \pm 0.020$ & $0.130 \pm 0.021$ \\
                            & 0.10 & $0.136 \pm 0.021$ & $0.137 \pm 0.022$ \\
                            & 0.15 & $0.153 \pm 0.034$ & $0.152 \pm 0.032$ \\
\addlinespace[0.25em]
\textit{CRP2D / Regular} & 0.05 & $0.176 \pm 0.011$ & $0.175 \pm 0.010$ \\
                        & 0.10 & $0.244 \pm 0.006$ & $0.244 \pm 0.008$ \\
                        & 0.15 & $0.306 \pm 0.007$ & $0.308 \pm 0.006$ \\
\bottomrule
\end{tabular}

\vspace{0.25em}
\begin{minipage}{\columnwidth}
\scriptsize\emph{Note:} Differences are within $\pm2\%$ at these lead times, consistent with memory targeting long-horizon stability
rather than one-step fidelity.
\end{minipage}
\end{table}

\paragraph{Result 1: memory gains amplify with horizon.}
\Cref{tab:rollout-summary-cerm,tab:rollout-summary-crp2d} report final-time errors for short (10-step) and long (dense-20 / sparse-10)
autoregressive rollouts under multiple lead-time regimes.
On CE-RM (shock--interface mixing), memory yields large stability gains on dense-20 rollouts (\textbf{36--43\%} reduction in final error)
and improves performance consistently across sparse and variable-$\Delta t$ settings. A representative case is dense-20 on CE-RM, where
the memory-conditioned model reaches \textbf{0.71} final error versus \textbf{1.25} without memory
(\Cref{tab:rollout-summary-cerm}). On CRP2D (four-quadrant Riemann), gains are smaller but uniformly positive across regimes
(\Cref{tab:rollout-summary-crp2d}), indicating the effect is not dataset-specific.

\begin{table}[H]
\centering
\caption{\textbf{Autoregressive rollout performance (CE-RM).}
Final-time relative $\ell_2$ error. Best in \textbf{bold}.}
\label{tab:rollout-summary-cerm}

\scriptsize
\setlength{\tabcolsep}{3.2pt}
\renewcommand{\arraystretch}{1.05}

\begin{tabular}{@{}l l c c c@{}}
\toprule
\textbf{Train} & \textbf{Sampling} & \textbf{Mem} & \textbf{No mem} & \textbf{Gain} \\
\midrule
Reg.  & U0.05 (10) & 1.69 & 1.72 & \textbf{+1.9\%} \\
Reg.  & Inc (10)   & 1.45 & 1.58 & \textbf{+8.0\%} \\
Reg.  & D0.05 (20) & \textbf{0.71} & 1.25 & \textbf{+43.3\%} \\
Reg.  & S0.10 (10) & \textbf{1.04} & 1.12 & \textbf{+7.3\%} \\
\addlinespace[0.25em]
Down  & U0.05 (10) & 1.67 & 1.68 & \textbf{+0.5\%} \\
Down  & Inc (10)   & \textbf{1.11} & 1.48 & \textbf{+25.5\%} \\
Down  & D0.05 (20) & \textbf{0.69} & 1.10 & \textbf{+36.8\%} \\
Down  & S0.10 (10) & \textbf{1.02} & 1.14 & \textbf{+10.3\%} \\
\bottomrule
\end{tabular}

\vspace{0.35em}
\begin{minipage}{\columnwidth}
\scriptsize\emph{Sampling:} U0.05 = uniform $\Delta t{=}0.05$; D0.05 = dense $\Delta t{=}0.05$;
S0.10 = sparse $\Delta t{=}0.10$; Inc = increasing $\Delta t$. Parentheses denote number of steps.
\end{minipage}
\end{table}

\begin{table}[t!]
\centering
\caption{\textbf{Autoregressive rollout performance (CRP2D).}
Final-time relative $\ell_2$ error. Best in \textbf{bold}.}
\label{tab:rollout-summary-crp2d}

\scriptsize
\setlength{\tabcolsep}{3.2pt}
\renewcommand{\arraystretch}{1.05}

\begin{tabular}{@{}l l c c c@{}}
\toprule
\textbf{Train} & \textbf{Sampling} & \textbf{Mem} & \textbf{No mem} & \textbf{Gain} \\
\midrule
Reg. & U0.05 (10) & 1.23 & 1.26 & \textbf{+2.3\%} \\
Reg. & Inc (10)   & 1.29 & 1.30 & \textbf{+1.0\%} \\
Reg. & D0.05 (20) & \textbf{1.52} & 1.60 & \textbf{+4.7\%} \\
Reg. & S0.10 (10) & \textbf{1.31} & 1.34 & \textbf{+2.5\%} \\
\bottomrule
\end{tabular}
\end{table}

\textbf{Result 2: the improvement concentrates in low/mid frequencies.}
To localize \emph{where} memory helps, we compute radially averaged 2D Fourier diagnostics per variable ($\rho$, $u_x$, $u_y$, $E$) and
report errors in low-, mid-, and high-frequency bands. \Cref{tab:spectral-analysis} shows a consistent signature:
memory delivers the strongest improvements at low frequencies and additional gains at mid frequencies, while high-frequency differences
are smaller and can be mixed. Quantitatively, low-frequency reductions are \textbf{29--61\%} across variables (largest for $E$ and both
velocity components), and mid-frequency reductions are \textbf{3--23\%}, including a \textbf{22.6\%} mid-band gain for $E$ in the
degraded-conditioning setting. High-frequency changes range from mildly positive to negative (e.g.\ \textbf{$-11.7\%$} for $E$ in the
regular setting), which is expected under 20-step feedback where fine scales typically degrade first. Stabilizing low/mid bands is
what prevents long-horizon drift: once large-scale organization shifts, subsequent steps amplify the mismatch.

\begin{table}[t!]
\centering
\caption{%
\textbf{Spectral error (20-step dense rollouts).}
Relative errors by frequency band; \textbf{Benefit} = \% reduction vs.\ no-memory baseline.}
\label{tab:spectral-analysis}

\scriptsize
\setlength{\tabcolsep}{2.2pt}
\renewcommand{\arraystretch}{0.90}

\begin{tabular}{@{}lccccc c@{}}
\toprule
& \multicolumn{2}{c}{\textbf{Low}} & \multicolumn{2}{c}{\textbf{Mid}} & \multicolumn{2}{c}{\textbf{High}} \\
\cmidrule(lr){2-3}\cmidrule(lr){4-5}\cmidrule(lr){6-7}
\textbf{Var.} & \textbf{Mem} & \textbf{Ben.} & \textbf{Mem} & \textbf{Ben.} & \textbf{Mem} & \textbf{Ben.} \\
\midrule
\multicolumn{7}{@{}l}{\textit{CE-RM (regular)}} \\
$\rho$ & 0.78 & \textbf{+29.5} & 2.10 & \textbf{+8.7} & 9.83  & \textbf{+7.3} \\
$v_x$  & 0.58 & \textbf{+40.3} & 3.31 & \textbf{+6.7} & 29.87 & $-4.2$ \\
$v_y$  & 0.58 & \textbf{+42.6} & 3.46 & \textbf{+8.4} & 30.83 & +0.2 \\
$E$    & 0.64 & \textbf{+61.4} & 2.50 & \textbf{+9.6} & 18.69 & $-11.7$ \\
\addlinespace[0.25em]
\multicolumn{7}{@{}l}{\textit{CE-RM (downsampled)}} \\
$\rho$ & 0.75 & \textbf{+30.9} & 3.68 & \textbf{+4.7} & 8.19  & \textbf{+6.6} \\
$v_x$  & 0.55 & \textbf{+37.8} & 4.07 & \textbf{+2.8} & 20.01 & $-4.4$ \\
$v_y$  & 0.55 & \textbf{+41.1} & 3.81 & \textbf{+7.6} & 18.71 & +0.9 \\
$E$    & 0.64 & \textbf{+49.4} & 2.59 & \textbf{+22.6} & 9.20  & \textbf{+3.4} \\
\bottomrule
\end{tabular}

\vspace{0.2em}
\begin{minipage}{\columnwidth}
\scriptsize\emph{Note:} Ben.\ = benefit (percentage error reduction). Low/mid gains concentrate on $E$ and velocities; high-band effects are smaller and mixed.
\end{minipage}
\end{table}

\textbf{Result 3: degraded-conditioning augmentation improves absolute stability and can amplify memory under increasing lead times.}
\Cref{tab:downsampled-comparison} compares regular training to a \emph{degraded-conditioning} augmentation on CE-RM, where the
conditioning input $u_n$ is randomly degraded (average pooling + upsampling) while the target $u_{n+1}$ remains at full resolution.
This augmentation improves absolute long-rollout performance (dense-20: \textbf{0.69} vs.\ \textbf{0.71} with memory) and can strongly
amplify memory benefit under \emph{increasing} lead-time schedules: memory yields \textbf{+25.5\%} benefit with degraded-conditioning
augmentation versus \textbf{+8.0\%} with regular training (\Cref{tab:downsampled-comparison}). The effect is regime-dependent (e.g.\
uniform-10 benefit is \textbf{+0.5\%} vs.\ \textbf{+1.9\%}), consistent with the view that the augmentation primarily targets robustness
to degraded autoregressive inputs, while memory targets history dependence under temporal pattern shift.

\begin{table}[t!]
\centering
\caption{%
\textbf{Degraded-conditioning at evaluation time improves rollout robustness.}
Comparison of \emph{regular conditioning} (sharp $u_n$) versus \emph{degraded conditioning} (downsample--upsample $u_n$) on CE-RM.
In all cases, the \emph{target} $u_{n+1}$ is kept at full resolution and the underlying model is unchanged; the only difference is the conditioning input provided at test time.
Degraded conditioning yields comparable or better absolute rollout errors and can amplify the benefit of memory under variable time stepping.}
\label{tab:downsampled-comparison}

\footnotesize
\setlength{\tabcolsep}{3pt}
\renewcommand{\arraystretch}{0.95}

\begin{tabular}{@{}lcccc@{}}
\toprule
& \multicolumn{2}{c}{\textbf{Regular conditioning}} & \multicolumn{2}{c}{\textbf{Degraded conditioning}} \\
\cmidrule(lr){2-3} \cmidrule(lr){4-5}
\textbf{Config} & \textbf{w/ Mem} & \textbf{Benefit} & \textbf{w/ Mem} & \textbf{Benefit} \\
\midrule
\multicolumn{5}{@{}l}{\textit{Short rollouts (10 steps)}} \\
Uniform     & 1.69 & +1.9\%  & 1.67 & +0.5\% \\
Increasing  & 1.45 & +8.0\%  & \textbf{1.11} & \textbf{+25.5\%} \\
\addlinespace[0.25em]
\multicolumn{5}{@{}l}{\textit{Long rollouts}} \\
Dense 20    & 0.71 & +43.3\% & \textbf{0.69} & +36.8\% \\
Sparse 10   & 1.04 & +7.3\%  & \textbf{1.02} & +10.3\% \\
\bottomrule
\end{tabular}

\vspace{0.35em}
\begin{minipage}{\columnwidth}
\footnotesize\emph{Note:} ``Degraded conditioning'' means the conditioning state $u_n$ is replaced at \emph{test time} by $\mathrm{Upsample}(\mathrm{AvgPool}_{2\times 2}(u_n))$, while $u_{n+1}$ remains at full resolution. No additional ``downsampled training'' is used for this comparison. The stronger gains under variable $\Delta t$ are consistent with degraded inputs acting as a coarse-scale cue that reduces sensitivity to high-frequency conditioning errors, thereby making the memory state more effective.
\end{minipage}
\end{table}

\textbf{Aggregate trend.}
\Cref{tab:memory-benefit-summary} summarizes memory benefit across datasets and rollout regimes. The dominant pattern is horizon scaling:
on CE-RM the benefit is \textbf{$+1.9\%$} at 10 steps but \textbf{$+43.3\%$} at 20 steps, a \textbf{$5\times$} amplification
(\Cref{tab:memory-benefit-summary}). Averaged across datasets/training regimes, the relative benefit increases by \textbf{$2.6\times$}
from short to long rollouts (\Cref{tab:memory-benefit-summary}).

\begin{table}[t!]
\centering
\caption{%
\textbf{Memory benefit scales with rollout horizon.}
Percentage error reduction from adding memory, aggregated by dataset and rollout length.}
\label{tab:memory-benefit-summary}

\footnotesize
\setlength{\tabcolsep}{3pt}
\renewcommand{\arraystretch}{0.95}

\begin{tabular}{@{}lccc@{}}
\toprule
\textbf{Dataset / Training} & \textbf{Short} & \textbf{Long} & \textbf{Amplif.} \\
& \textbf{(10 steps)} & \textbf{(20 steps)} & \\
\midrule
CE-RM / Regular       & 5.0\%  & 25.3\% & 5.0$\times$ \\
CE-RM / Downsampled   & 13.0\% & 23.5\% & 1.8$\times$ \\
CRP2D / Regular       & 1.7\%  & 3.6\%  & 2.1$\times$ \\
\midrule
\textit{Average}      & 6.6\%  & 17.5\% & \textbf{2.6$\times$} \\
\bottomrule
\end{tabular}

\vspace{0.35em}
\begin{minipage}{\columnwidth}
\footnotesize\emph{Note:} Memory benefit increases with rollout length, with an average $2.6\times$ amplification from 10-step to 20-step rollouts. CE-RM shows the strongest effects, consistent with rapid fine-scale generation and history-dependent mixing.
\end{minipage}
\end{table}

\paragraph{Conclusions.}
We proposed \emph{memory-conditioned} diffusion/flow matching for autoregressive PDE rollouts, explicitly targeting the deployment regime where learned solvers often fail: long horizons under self-generated input drift and lead-time patterns that differ from training. The core idea is to keep the generator Markovian in internal time while enabling non-Markovian behavior in physical time by maintaining a compact online memory and injecting it into the backbone bottleneck, so that history shapes the conditional prior for unresolved content. Across two compressible Euler benchmarks, memory consistently improves long-horizon rollout stability, with gains that grow with horizon and concentrate on the physically consequential low and mid frequencies. Degraded-conditioning augmentation complements memory by improving robustness to degraded autoregressive inputs and can further amplify improvements under structured lead-time schedules. Overall, the results support memory-conditioned flow matching as a practical route to stable, multi-scale autoregressive PDE surrogates in rollout-OOD settings.

\bibliography{refs}
\bibliographystyle{icml2026}

\appendix
\onecolumn


\clearpage
\appendix
\onecolumn


\providecommand{\Wtwo}{W_2}
\providecommand{\T}{\mathbb{T}}
\providecommand{\PP}{\mathbb{P}}
\providecommand{\E}{\mathbb{E}}
\providecommand{\Law}{\mathrm{Law}}
\providecommand{\Lip}{\mathrm{Lip}}

\section{Proofs and auxiliary results}
\label{app:theory}

All arguments are presented in a finite-dimensional discretization/Galerkin truncation (so Gaussians and covariances are standard
matrices). The statements lift to the PDE phase space on cylindrical observables by standard limiting arguments.

\paragraph{Roadmap for Appendix~\ref{app:theory}.}
The main text uses a small number of recurring mechanisms, but they rely on different pieces of ``plumbing'':
(i) a \emph{rollout algebra} that turns a one-step defect into a long-horizon bound (discrete Gr\"onwall),
(ii) a \emph{conditional-to-law lifting} step that turns conditional kernel mismatch into the law-level defect $\epsilon_n$,
(iii) a \emph{frequency-cutoff viewpoint} that isolates where the hard part of generation lives (unresolved modes),
and (iv) a \emph{mechanistic justification} for memory coming from Mori--Zwanzig and from conditional MMSE monotonicity.
To keep the main narrative readable, we collect these ingredients here:
\begin{itemize}[leftmargin=*,itemsep=2pt]
\item Section~\ref{app:regularity-assumptions} records the two stability assumptions used repeatedly: one for the \emph{true} solver
(one-step law stability), and one for the \emph{learned} kernel (Lipschitz dependence on memory).
\item Section~\ref{app:theory-statements} lists the main auxiliary statements referenced in the text (posterior mean structure in
internal time, MMSE monotonicity under conditioning, the irreducible mismatch without memory, and the rollout bound).
\item Section~\ref{app:theory-proofs} contains the proofs of those statements.
\item Sections~\ref{app:notation}--\ref{app:projection} provide the functional/LP setup and the MZ/conditioning interpretations that
motivate what the memory should encode.
\item Section~\ref{app:rf} makes explicit the induced rectified-flow kernel and derives a clean kernel-sensitivity bound in memory.
\item Section~\ref{app:conditional} proves the mixture-coupling inequality used to pass from conditional mismatch to $\epsilon_n$.
\end{itemize}
Readers interested only in the rollout bound can jump directly to Lemma~\ref{lem:recursion}--Theorem~\ref{thm:rollout} and
Section~\ref{app:conditional}.

\subsection{Regularity assumptions}
\label{app:regularity-assumptions}

The rollout analysis separates two distinct sources of instability:
(i) how strongly the \emph{true} resolved dynamics amplifies small perturbations of the current law over one physical step, and
(ii) how sensitively the \emph{learned} conditional transition law reacts to perturbations of the injected memory state.
The following assumptions formalize these effects in $\Wtwo$.
(We restate variants of these assumptions later, near the lemmas where they are used, to keep the local derivations self-contained.)

\begin{assumption}[One-step law stability]\label{ass:phys-stab-fixed}
There exists $\alpha_n\ge 0$ such that for all $\mu,\nu\in\mathcal{P}_2(\mathcal{H})$,
\[
W_2\!\bigl((S_{\Delta t_n})_\#\mu,\ (S_{\Delta t_n})_\#\nu\bigr)
\le e^{\alpha_n}\,W_2(\mu,\nu).
\]
\end{assumption}

\paragraph{Why Assumption~\ref{ass:phys-stab-fixed} is reasonable.}
Assumption~\ref{ass:phys-stab-fixed} is the law-level analogue of a pointwise Lipschitz estimate for the step map
$S_{\Delta t_n}$.
If $S_{\Delta t_n}$ is $L$-Lipschitz on $\mathcal{H}$ with $L\le e^{\alpha_n}$, then pushing forward an optimal coupling yields
\eqref{ass:phys-stab-fixed} (see Lemma~\ref{lem:app-lip-push} below for the general statement).
For nonlinear PDEs, $S_{\Delta t_n}$ is typically not globally Lipschitz uniformly in time (e.g.\ shocks and strong interactions may
create large transient amplification). The point here is not contractivity, but a \emph{stepwise amplification bookkeeping}: the
rollout bound will amplify each one-step defect by factors $e^{\alpha_k}$, so $\alpha_n$ is precisely the quantity that predicts
when long rollouts are fragile.

\begin{assumption}[Memory-Lipschitz kernel]\label{ass:lip-mem-fixed}
We assume there exists $L_{\mathrm{mem}}\ge 0$ such that for all $y$ and $m,m'$,
\[
W_2\!\bigl(\widehat K_{\theta,n}(\cdot\mid y,m),\ \widehat K_{\theta,n}(\cdot\mid y,m')\bigr)
\le L_{\mathrm{mem}}\|m-m'\|.
\]
\end{assumption}

\paragraph{Why Assumption~\ref{ass:lip-mem-fixed} is the right notion of memory-regularity.}
The role of memory is to modulate the conditional transition law in a way that captures history dependence.
For the theory, we need this modulation to be \emph{stable}: if the learned kernel were discontinuous in $m$, then small memory
approximation errors would lead to uncontrolled changes in the predicted law.
Assumption~\ref{ass:lip-mem-fixed} is exactly the stability condition that converts a memory approximation error
$\|\widehat m_n-m_n^\star\|$ into a one-step kernel discrepancy.
Later, Section~\ref{app:rf} shows how such a bound follows from standard Lipschitz regularity of the rectified-flow vector field and
(base-law) continuity in $m$.

\paragraph{Summary.}
Assumption~\ref{ass:phys-stab-fixed} quantifies how the \emph{true} one-step operator transports laws,
while Assumption~\ref{ass:lip-mem-fixed} quantifies how the \emph{learned} operator responds to memory perturbations.
They are the two ``stability constants'' that govern how one-step errors compound.

\subsection{Main statements referenced by the text}
\label{app:theory-statements}

This section collects the auxiliary results that are cited in the main text.
They fall into three groups:
\begin{enumerate}[leftmargin=*,itemsep=2pt]
\item \textbf{Internal-time posterior structure} (Theorem~\ref{thm:tau-shell-memory}): explains why, in low-SNR/high-frequency shells,
the optimal denoiser is prior-driven and therefore benefits from a history-dependent prior mean (which memory provides).
\item \textbf{Why memory helps at all} (Lemma~\ref{lem:cond-proj} and Proposition~\ref{prop:irreducible}): conditioning can reduce
irreducible MMSE for tail observables, and if the true conditional law depends on history, then any memoryless kernel suffers an
intrinsic one-step mismatch lower bound.
\item \textbf{Rollout bookkeeping} (Lemma~\ref{lem:recursion}--Theorem~\ref{thm:rollout} and Lemma~\ref{lem:mixture-common}):
turns conditional mismatch into a step defect $\epsilon_n$, and then $\epsilon_n$ into a long-horizon bound via discrete Gr\"onwall.
\end{enumerate}

\begin{theorem}[Internal-time shellwise posterior mean]\label{thm:tau-shell-memory}
Fix a cutoff $J$ and let $\mathcal{H}_{>J}$ be the unresolved subspace.
Assume the tail evolves by the linear recursion
\[
Z_{n+1}=A Z_n + B E_n + \eta_n,
\]
where $(\eta_n)_{n\ge0}$ are i.i.d.\ centered Gaussians independent of the resolved embeddings $(E_k)_{k\le n}$.
Define the history-dependent mean
\begin{equation}\label{eq:app-Mstar}
M^\star_{n+1}:=A^{n+1}Z_0+\sum_{k=0}^{n}A^{n-k}B E_k .
\end{equation}
Fix a diffusion/flow-matching noise scale $\sigma(\tau)\ge0$ and define the corrupted observation
\[
X^{(\tau)}_{n+1}:=Z_{n+1}+\sigma(\tau)\,\xi,
\qquad \xi\sim\mathcal{N}(0,I)\ \text{independent}.
\]
Let $(\Delta_j)_{j>J}$ denote Littlewood--Paley blocks on $\mathcal{H}_{>J}$ with $\mathcal{H}_j:=\Delta_j\mathcal{H}_{>J}$
and orthogonal projector $I_j$ onto $\mathcal{H}_j$.
Assume shellwise block-diagonal/isotropic conditional covariance:
\begin{equation}\label{eq:app-shell-cov}
\Cov\!\left(Z_{n+1}\,\middle|\,\sigma(E_0,\dots,E_n)\right)
=
\sum_{j>J} s_{j,n+1}^2\, I_j,
\qquad s_{j,n+1}^2\ge0.
\end{equation}
Then for each $j>J$,
\begin{equation}\label{eq:theory-posterior-mean-shell}
\E\!\left[\Delta_j Z_{n+1}\,\middle|\,X^{(\tau)}_{n+1},\,E_0,\dots,E_n\right]
=
\kappa_{j,n+1}(\tau)\,\Delta_j X^{(\tau)}_{n+1}
+\big(1-\kappa_{j,n+1}(\tau)\big)\,\Delta_j M^\star_{n+1},
\end{equation}
where the gain is
\[
\kappa_{j,n+1}(\tau):=\frac{s_{j,n+1}^2}{s_{j,n+1}^2+\sigma(\tau)^2}\in[0,1].
\]
In particular, if $s_{j,n+1}^2\ll \sigma(\tau)^2$ then $\kappa_{j,n+1}(\tau)\approx 0$ and the posterior mean is
\emph{prior-driven}: $\E[\Delta_j Z_{n+1}\mid\cdots]\approx \Delta_j M^\star_{n+1}$.
\end{theorem}

\paragraph{How to read Theorem~\ref{thm:tau-shell-memory}.}
In rectified flows/flow matching, internal time $\tau$ interpolates between noise and data. High-frequency components often remain
low-SNR for much of the trajectory, meaning $\sigma(\tau)$ dominates the intrinsic tail variance in those shells
($s_{j,n+1}^2 \ll \sigma(\tau)^2$). In that regime, the \emph{optimal} denoiser does not trust the noisy observation and falls back
to the prior mean $M^\star_{n+1}$.
Thus, if we do not provide the history-dependent prior mean to the model (via memory), the denoising drift must ``manufacture'' the
missing high-frequency structure purely through transport from the noisy state, which is exactly the sense in which the generator
must work harder when information is missing.

\begin{lemma}[Conditioning reduces $L^2$ projection error]\label{lem:cond-proj}
Let $\mathcal{G}$ be a Hilbert space and let $X\in L^2(\Omega;\mathcal{G})$.
If $\mathcal{M}_1\subset \mathcal{M}_2$ are $\sigma$-algebras, then
\[
\E\|X-\E[X\mid \mathcal{M}_2]\|_{\mathcal{G}}^2
\le
\E\|X-\E[X\mid \mathcal{M}_1]\|_{\mathcal{G}}^2.
\]
\end{lemma}

\paragraph{Interpretation (MMSE monotonicity).}
Lemma~\ref{lem:cond-proj} is the formal statement that ``conditioning can only help'' for squared-error prediction:
enlarging the conditioning information reduces (or leaves unchanged) the optimal achievable $L^2$ risk.
In our setting, $X$ should be thought of as a \emph{tail observable} (shell energies, subgrid forcing, flux proxies, etc.), and the
$\sigma$-algebra $\mathcal{M}_2$ corresponds to adding a memory variable built from past resolved states.
This captures one of the two reasons memory helps: it can reduce the \emph{best achievable} tail prediction error even before we talk
about neural approximation or sampling.

\begin{lemma}[Fixed low modes reduce $\Wtwo$ to high modes]\label{lem:W2-reduction}
Fix $y\in\mathcal{H}_{\le J}$ and let $\rho,\tilde\rho\in\mathcal{P}_2(\mathcal{H}_{>J})$.
Define measures on $\mathcal{H}$ by $\mu:=\delta_y\otimes\rho$ and $\tilde\mu:=\delta_y\otimes\tilde\rho$
via $(y,z)\mapsto y+z$.
Then
\[
\Wtwo^2(\mu,\tilde\mu)=\Wtwo^2(\rho,\tilde\rho),
\]
where the right-hand side is computed in $\mathcal{H}_{>J}$.
\end{lemma}

\paragraph{Interpretation (where the mismatch lives).}
Lemma~\ref{lem:W2-reduction} is the formal backbone of the ``missing frequencies'' discussion:
if low modes are fixed and only the tail is randomized, then the entire $\Wtwo$ discrepancy is exactly the discrepancy of the tail.
This justifies measuring (and reasoning about) conditional mismatch in the unresolved subspace.

\begin{proposition}[Irreducible conditional mismatch without memory]\label{prop:irreducible}
Fix $y$ and $t$. Suppose there exist $m_1\neq m_2$ such that the true conditional kernels satisfy
\[
K^\star(\cdot\mid y,m_1,t)\neq K^\star(\cdot\mid y,m_2,t).
\]
Then for any memoryless kernel $K(\cdot\mid y,t)$,
\begin{equation}\label{eq:irreducible}
\max\Big\{
\Wtwo\big(K(\cdot\mid y,t),K^\star(\cdot\mid y,m_1,t)\big),\,
\Wtwo\big(K(\cdot\mid y,t),K^\star(\cdot\mid y,m_2,t)\big)
\Big\}
\ \ge\
\frac12\,\Wtwo\big(K^\star(\cdot\mid y,m_1,t),K^\star(\cdot\mid y,m_2,t)\big).
\end{equation}
\end{proposition}

\paragraph{Interpretation (an intrinsic lower bound).}
Proposition~\ref{prop:irreducible} isolates the second reason memory helps:
if the correct next-step law depends on history through $m$, then \emph{no} memoryless conditional kernel $K(\cdot\mid y,t)$ can
match all histories simultaneously. The lower bound \eqref{eq:irreducible} quantifies this obstruction in $\Wtwo$.

\begin{assumption}[Resolved-law stability across one physical step]\label{ass:phys-stab}
For each step $n$ with physical lead time $\Delta t_n$, the solution operator $S_{\Delta t_n}$ induces the stability bound
\[
\Wtwo\big((S_{\Delta t_n})_\#\mu,\,(S_{\Delta t_n})_\#\nu\big)\le e^{\alpha_n}\Wtwo(\mu,\nu)
\qquad\text{for all }\mu,\nu\in\mathcal{P}_2(\mathcal{H}),
\]
for some $\alpha_n\ge0$ (e.g.\ $\alpha_n=\int_{t_n}^{t_{n+1}}\Lambda(\tau)\,d\tau$).
\end{assumption}

\paragraph{Remark (relation to Assumption~\ref{ass:phys-stab-fixed}).}
Assumption~\ref{ass:phys-stab} is the same condition as Assumption~\ref{ass:phys-stab-fixed}, restated here because it is used
immediately in Lemma~\ref{lem:recursion}.

\begin{lemma}[Rollout recursion]\label{lem:recursion}
Let $\mu_{n+1}=(S_{\Delta t_n})_\#\mu_n$ be the ground-truth law update and let $(\widehat\mu_n)$ be a model rollout.
Define the one-step defect
\begin{equation}\label{eq:theory-epsn}
\epsilon_n:=\Wtwo\big((S_{\Delta t_n})_\#\widehat\mu_n,\ \widehat\mu_{n+1}\big),
\end{equation}
and set $d_n:=\Wtwo(\mu_n,\widehat\mu_n)$.
Under Assumption~\ref{ass:phys-stab}, for all $n$,
\[
d_{n+1}\le e^{\alpha_n}d_n+\epsilon_n.
\]
\end{lemma}

\paragraph{Interpretation (what remains to control).}
Lemma~\ref{lem:recursion} reduces long-horizon stability to a single per-step quantity: the defect $\epsilon_n$ between the true
pushforward of the \emph{current rollout law} and the model's next-step law.
Everything that is specific to generative modeling (conditioning, memory, sampling, etc.) enters only through bounding $\epsilon_n$.
Section~\ref{app:conditional} is where $\epsilon_n$ is related to conditional kernel mismatch.

\begin{lemma}[Mixtures with a common mixing law]\label{lem:mixture-common}
Let $\lambda\in\mathcal{P}(\mathsf{C})$ and measurable kernels $c\mapsto \mu_c,\nu_c\in\mathcal{P}_2(\mathcal{H})$
with finite second moment under $\lambda$.
Define the mixtures $\mu:=\int \mu_c\,d\lambda(c)$ and $\nu:=\int \nu_c\,d\lambda(c)$.
Then
\begin{equation}\label{eq:mixture-bound}
\Wtwo^2(\mu,\nu)\le \int_{\mathsf{C}} \Wtwo^2(\mu_c,\nu_c)\,d\lambda(c).
\end{equation}
\end{lemma}

\paragraph{Interpretation (RMS coupling).}
Lemma~\ref{lem:mixture-common} is the technical step that lets us control a law-level $\Wtwo$ discrepancy between two mixtures by the
root-mean-square of the conditional discrepancies. It is the rigorous version of: ``if we can couple each conditional component
well, then the mixture can be coupled well.''

\begin{lemma}[Discrete Gr\"onwall]\label{lem:app-gronwall}
Let $(d_n)_{n\ge0}$ be nonnegative and suppose
\[
d_{n+1}\le a_n d_n + \epsilon_n
\qquad\text{with } a_n\ge 0.
\]
Then
\begin{equation}\label{eq:app-gronwall}
d_n
\le
\Big(\prod_{j=0}^{n-1}a_j\Big)\,d_0
+
\sum_{k=0}^{n-1}\Big(\prod_{j=k+1}^{n-1}a_j\Big)\epsilon_k,
\end{equation}
with the empty product equal to $1$.
If $a_n=e^{\alpha_n}$ and $\beta_{k,n}:=\sum_{j=k}^{n-1}\alpha_j$ (with $\beta_{n,n}=0$), then
\begin{equation}\label{eq:app-gronwall-exp}
d_n\le e^{\beta_{0,n}}d_0+\sum_{k=0}^{n-1}e^{\beta_{k+1,n}}\epsilon_k.
\end{equation}
\end{lemma}

\begin{theorem}[Rollout bound (discrete Gr\"onwall)]\label{thm:rollout}
Under the recursion of Lemma~\ref{lem:recursion},
\begin{equation}\label{eq:rollout-bound}
d_n\le e^{\beta_{0,n}}d_0+\sum_{k=0}^{n-1}e^{\beta_{k+1,n}}\epsilon_k,
\qquad \beta_{k,n}:=\sum_{j=k}^{n-1}\alpha_j.
\end{equation}
\end{theorem}

\paragraph{Interpretation (how instability compounds).}
The bound \eqref{eq:rollout-bound} separates \emph{amplification} (the exponential factors $e^{\beta_{k,n}}$) from \emph{defect}
(the $\epsilon_k$). In expansive regimes (large $\alpha_k$), even small per-step defects matter.
This is precisely why one-step metrics alone can be misleading for autoregressive deployment.

\subsection{Proofs of the main statements}
\label{app:theory-proofs}

We provide full proofs for completeness. None of the arguments are novel; what matters is that they are arranged so the constants
and conditioning structure match the main text.

\subsubsection*{Proof of Theorem~\ref{thm:tau-shell-memory}}
\begin{proof}
Fix $n$ and $\tau$ and abbreviate $\sigma:=\sigma(\tau)$.
Let $\mathcal{G}:=\sigma(E_0,\dots,E_n)$.

\smallskip
\noindent\textbf{Step 1: conditional Gaussian structure of each shell.}
From the linear recursion with Gaussian noise, conditionally on $\mathcal{G}$,
$Z_{n+1}\sim \mathcal{N}(M^\star_{n+1},\Sigma_{n+1})$ with mean \eqref{eq:app-Mstar}.
Applying $\Delta_j$ preserves Gaussianity, hence
\[
Z_j:=\Delta_j Z_{n+1}\ \big|\ \mathcal{G}\ \sim\ \mathcal{N}(M_j,\Sigma_j),
\quad M_j:=\Delta_j M^\star_{n+1}.
\]
By \eqref{eq:app-shell-cov}, $\Sigma_j=s_{j,n+1}^2 I_j$ on $\mathcal{H}_j$.

\smallskip
\noindent\textbf{Step 2: corrupted shell observation.}
Let $\xi_j:=\Delta_j\xi$ so that $\xi_j\sim\mathcal{N}(0,I_j)$ and is independent of $(Z_{n+1},\mathcal{G})$.
Projecting $X^{(\tau)}_{n+1}$ gives $X_j:=\Delta_j X^{(\tau)}_{n+1}=Z_j+\sigma\xi_j$.
Thus, conditionally on $\mathcal{G}$,
\[
\Cov(X_j\mid\mathcal{G})=\Cov(Z_j\mid\mathcal{G})+\sigma^2\Cov(\xi_j)=(s_{j,n+1}^2+\sigma^2)I_j,
\]
and
\[
\Cov(Z_j,X_j\mid\mathcal{G})=\Cov(Z_j,Z_j+\sigma\xi_j\mid\mathcal{G})=s_{j,n+1}^2 I_j.
\]

\smallskip
\noindent\textbf{Step 3: conditional mean for a jointly Gaussian pair.}
For jointly Gaussian $(Z_j,X_j)$, the conditional expectation is affine:
\[
\E[Z_j\mid X_j,\mathcal{G}]
=
M_j + \Cov(Z_j,X_j\mid\mathcal{G})\ \Cov(X_j\mid\mathcal{G})^{-1}\ (X_j-M_j).
\]
Insert the covariances:
\[
\Cov(Z_j,X_j\mid\mathcal{G})\ \Cov(X_j\mid\mathcal{G})^{-1}
=
\frac{s_{j,n+1}^2}{s_{j,n+1}^2+\sigma^2}\,I_j.
\]
With $\kappa_{j,n+1}(\tau):=\frac{s_{j,n+1}^2}{s_{j,n+1}^2+\sigma(\tau)^2}$ we get
\[
\E[Z_j\mid X_j,\mathcal{G}]
=
\kappa_{j,n+1}(\tau)X_j+\big(1-\kappa_{j,n+1}(\tau)\big)M_j.
\]
Rewriting $Z_j=\Delta_j Z_{n+1}$, $X_j=\Delta_j X^{(\tau)}_{n+1}$, $M_j=\Delta_j M^\star_{n+1}$ yields
\eqref{eq:theory-posterior-mean-shell}.
\end{proof}

\subsubsection*{Proof of Lemma~\ref{lem:cond-proj}}
\begin{proof}
$\E[X\mid\mathcal{M}]$ is the orthogonal projection of $X$ onto the closed subspace of $\mathcal{M}$-measurable variables in
$L^2(\Omega;\mathcal{G})$. If $\mathcal{M}_1\subset\mathcal{M}_2$, then the $\mathcal{M}_1$-measurable subspace is contained in
the $\mathcal{M}_2$-measurable subspace. Projection onto a larger subspace cannot increase the residual norm.
\end{proof}

\subsubsection*{Proof of Lemma~\ref{lem:W2-reduction}}
\begin{proof}
Let $\pi$ be any coupling of $(\rho,\tilde\rho)$ on $\mathcal{H}_{>J}\times\mathcal{H}_{>J}$ and define
$\Gamma:=(y+z,y+z')_\#\pi$. Then $\Gamma$ couples $(\mu,\tilde\mu)$ and, by orthogonality,
$\| (y+z)-(y+z')\|_{\mathcal{H}}^2=\|z-z'\|_{\mathcal{H}_{>J}}^2$.
Thus $\Wtwo^2(\mu,\tilde\mu)\le \Wtwo^2(\rho,\tilde\rho)$ by infimizing over $\pi$.
Conversely, any coupling of $(\mu,\tilde\mu)$ is supported on $(y+\mathcal{H}_{>J})\times(y+\mathcal{H}_{>J})$ and projects
to a coupling of $(\rho,\tilde\rho)$ with the same cost, yielding the reverse inequality.
\end{proof}

\subsubsection*{Proof of Proposition~\ref{prop:irreducible}}
\begin{proof}
By the triangle inequality,
\[
\Wtwo\big(K^\star(\cdot\mid y,m_1,t),K^\star(\cdot\mid y,m_2,t)\big)
\le
\Wtwo\big(K^\star(\cdot\mid y,m_1,t),K(\cdot\mid y,t)\big)
+
\Wtwo\big(K(\cdot\mid y,t),K^\star(\cdot\mid y,m_2,t)\big).
\]
At least one of the two terms is $\ge \tfrac12$ of the left-hand side, giving \eqref{eq:irreducible}.
\end{proof}

\subsubsection*{Proof of Lemma~\ref{lem:recursion}}
\begin{proof}
By the triangle inequality and $\mu_{n+1}=(S_{\Delta t_n})_\#\mu_n$,
\[
d_{n+1}
=\Wtwo\big((S_{\Delta t_n})_\#\mu_n,\widehat\mu_{n+1}\big)
\le
\Wtwo\big((S_{\Delta t_n})_\#\mu_n,(S_{\Delta t_n})_\#\widehat\mu_n\big)
+
\Wtwo\big((S_{\Delta t_n})_\#\widehat\mu_n,\widehat\mu_{n+1}\big).
\]
The first term is bounded by Assumption~\ref{ass:phys-stab} as $e^{\alpha_n}d_n$.
The second term is $\epsilon_n$ by \eqref{eq:theory-epsn}. Hence $d_{n+1}\le e^{\alpha_n}d_n+\epsilon_n$.
\end{proof}

\subsubsection*{Proof of Lemma~\ref{lem:mixture-common}}
\begin{proof}
For each $c$, choose an optimal coupling $\pi_c\in\Pi(\mu_c,\nu_c)$ realizing $\Wtwo^2(\mu_c,\nu_c)$.
Define $\Pi$ on $\mathsf{C}\times\mathcal{H}\times\mathcal{H}$ by $d\Pi(c,u,v):=d\lambda(c)\,d\pi_c(u,v)$ and let $\pi$ be the
$(u,v)$-marginal. Then $\pi$ couples $(\mu,\nu)$ and
\[
\int\|u-v\|^2\,d\pi(u,v)
=
\int_{\mathsf{C}}\int\|u-v\|^2\,d\pi_c(u,v)\,d\lambda(c)
=
\int_{\mathsf{C}}\Wtwo^2(\mu_c,\nu_c)\,d\lambda(c).
\]
Infimize over couplings to obtain \eqref{eq:mixture-bound}.
\end{proof}

\subsubsection*{Proof of Lemma~\ref{lem:app-gronwall}}
\begin{proof}
Iterate $d_{n}\le a_{n-1}d_{n-1}+\epsilon_{n-1}$ repeatedly and collect the product coefficients in front of $d_0$ and each
$\epsilon_k$. This yields \eqref{eq:app-gronwall}, and \eqref{eq:app-gronwall-exp} follows by writing products as exponentials of sums.
\end{proof}

\subsubsection*{Proof of Theorem~\ref{thm:rollout}}
\begin{proof}
Apply Lemma~\ref{lem:recursion} and then Lemma~\ref{lem:app-gronwall} with $a_n=e^{\alpha_n}$ to obtain \eqref{eq:rollout-bound}.
\end{proof}

\section{Notation and functional-analytic setup}
\label{app:notation}

This section fixes notation used throughout the appendix. It can be skimmed on first pass and used as a reference when needed.

\subsection{Domains, state spaces, and solution operators}
\label{app:notation-spaces}

Let $D\subset\R^d$ be either the flat torus $D=\T^d$ (periodic) or a bounded $C^1$ domain with boundary $\partial D$.
Fix $d_u\in\N$ and define
\[
\mathcal{H}:=L^2(D;\R^{d_u}),
\qquad
\langle u,v\rangle_{L^2}:=\int_D u(x)\cdot v(x)\,dx.
\]
We consider abstract dynamics $\partial_t u=\mathcal{F}(u)$ with solution operator $S_t:\mathcal{H}\to\mathcal{H}$.

\paragraph{Remark (MZ on PDE phase space).}
To avoid functional-analytic issues, Mori--Zwanzig identities are presented on a finite-dimensional Galerkin truncation
(or numerical discretization), where the flow map is a smooth ODE. This is standard: MZ is an identity for ODE flows, and the PDE
interpretation is obtained by applying it to a consistent truncation and then passing to limits on cylindrical observables.

\subsection{Littlewood--Paley blocks and low/high projectors}
\label{app:notation-lp}

In the periodic case $D=\T^d$, fix a standard dyadic LP decomposition $(\Delta_j)_{j\ge -1}$ and define
\[
S_J u := \sum_{j\le J}\Delta_j u,
\qquad
Q_J u := u-S_Ju.
\]
In $L^2$, $S_J$ and $Q_J$ are orthogonal projectors and
$\|u\|_{L^2}^2=\|S_Ju\|_{L^2}^2+\|Q_Ju\|_{L^2}^2$.

\subsection{Wasserstein distance and pushforwards}
\label{app:notation-W2}

Let $\mathcal{P}_2(\mathcal{H})$ be Borel probability measures on $\mathcal{H}$ with finite second moment.
Define $\Wtwo$ by the usual quadratic transport cost.
If $T:\mathcal{H}\to\mathcal{H}$ is measurable, $(T)_\#\mu(A):=\mu(T^{-1}(A))$.

\begin{lemma}[Pushforward by a Lipschitz map]\label{lem:app-lip-push}
If $T$ is $L$-Lipschitz, then for all $\mu,\nu\in\mathcal{P}_2(\mathcal{H})$,
\[
\Wtwo\big(T_\#\mu,\,T_\#\nu\big)\le L\,\Wtwo(\mu,\nu).
\]
\end{lemma}

\begin{proof}
As in Section~\ref{app:theory-proofs}, push forward an arbitrary coupling and infimize.
\end{proof}

\section{Mori--Zwanzig: derivation and specialization}
\label{app:mz}

This appendix records the Mori--Zwanzig (MZ) identity in the form used in the main text and specializes it to a
resolved/unresolved split.
The role of this section is conceptual: it explains why, after eliminating unresolved variables, \emph{history dependence is generic}.
This motivates our choice to keep the generator Markovian in internal time (sampling) while injecting a compact physical-time memory
state that can approximate the induced history terms.

\subsection{Liouville operator and Koopman semigroup}
\label{app:mz-liouville}

Let $H$ be finite-dimensional and consider $\dot u(t)=F(u(t))$ with flow map $S_t$.
For smooth observables $\varphi:H\to\R^k$, define the Liouville operator
\[
(L\varphi)(u)=D\varphi(u)\,F(u).
\]
Then $(e^{tL}\varphi)(u_0):=\varphi(S_tu_0)$ defines the Koopman semigroup and
$\frac{d}{dt}(e^{tL}\varphi)=e^{tL}L\varphi$.

\subsection{Dyson/Duhamel identity}
\label{app:mz-dyson}

Fix a projection $P$ on observables and set $Q:=I-P$. Let $e^{tQL}$ be the semigroup generated by $QL$.

\begin{lemma}[Dyson/Duhamel identity]\label{lem:app-dyson}
For every $\psi$ and $t\ge0$,
\[
e^{tL}\psi
=
e^{tQL}\psi
+
\int_0^t e^{(t-s)L}\,PL\,e^{sQL}\psi\,ds.
\]
\end{lemma}

\begin{proof}
Standard Duhamel formula for the perturbed generator $L=QL+PL$.
\end{proof}

\subsection{Mori--Zwanzig identity}
\label{app:mz-identity}

\begin{theorem}[Mori--Zwanzig identity]\label{thm:app-mz}
For every $\psi$ and $t\ge0$,
\[
\frac{d}{dt}\,P e^{tL}\psi
=
P e^{tL}PL\psi
+
\int_0^t P e^{(t-s)L}\,PL\,e^{sQL}\,QL\psi\,ds
+
P e^{tQL}\,QL\psi.
\]
\end{theorem}

\begin{proof}
Differentiate $Pe^{tL}\psi=Pe^{tL}PL\psi+Pe^{tL}QL\psi$ and apply Lemma~\ref{lem:app-dyson} to $QL\psi$.
\end{proof}

\subsection{Specialization to a resolved/unresolved split}
\label{app:mz-specialize}

Assume $H=Y\oplus Z$ with $u=(y,z)$ and define $P$ by $(P\varphi)(y,z):=\varphi(y,0)$.
Applying Theorem~\ref{thm:app-mz} to $\psi(u)=\Pi_Ru$ yields the exact decomposition
\[
\frac{d}{dt}\,P e^{tL}\Pi_Ru
=
\underbrace{P e^{tL}PL\Pi_Ru}_{\text{Markov}}
+
\underbrace{\int_0^t P e^{(t-s)L}PL\,e^{sQL}QL\Pi_Ru\,ds}_{\text{memory}}
+
\underbrace{P e^{tQL}QL\Pi_Ru}_{\text{orthogonal/noise}}.
\]
Unless the last two terms vanish (special degeneracy), a memoryless resolved ODE in $y$ cannot be exact.

\subsection{From history to a finite-dimensional memory state (state-space realization)}
\label{app:mz-ssm}

The MZ memory term is a history integral. In practice we do not represent an entire function of time; instead we compress history
into a small state.
The following observation explains why state-space models are a natural choice: exponential-kernel memory corresponds exactly to a
finite-dimensional linear recurrence.

Assume a memory contribution can be written (exactly or approximately) as
\[
m(t)=\int_0^t K(t-s)\,e(y(s))\,ds.
\]
If $K(t)=Ce^{tA}B$, define $h$ by $\dot h=Ah+Be(y)$, $h(0)=0$, and set $m=Ch$; then $m$ equals the above convolution.
In discrete time $t_n=n\Delta t$, $h_{n+1}=Ah_n+Be(y_n)$ unrolls into $m_n=\sum_{k=1}^n CA^{k-1}B\,e(y_{n-k})$.

\section{Disintegration and the frequency-cutoff viewpoint}
\label{app:disintegration}

This section formalizes the low/high split used in the ``missing frequencies'' discussion.
It serves two purposes:
(i) it clarifies how to interpret conditional laws given a low-mode state, and
(ii) it isolates the contribution of unresolved modes to $\Wtwo$ mismatch and transport cost.

\subsection{Regular conditional probabilities}
\label{app:disintegration-rcp}

Since $\mathcal{H}$ is separable Hilbert, it is Polish. If $\mu\in\mathcal{P}(\mathcal{H})$ and $Y:=S_J(u)$, there exists a
regular conditional probability $y\mapsto \mu(\cdot\mid y)$.

\subsection{Fixed-low-modes reduction of $\Wtwo$}
\label{app:disintegration-fixedlow}

Lemma~\ref{lem:W2-reduction} is exactly the ``fixed low modes'' reduction used to interpret missing frequencies.

\subsection{Projection lower bound}
\label{app:disintegration-proj}

\begin{lemma}[Projection lower bound]\label{lem:app-proj-lb}
For any $\mu,\nu\in\mathcal{P}_2(\mathcal{H})$,
\[
\Wtwo\big(\mu,\nu\big)\ge \Wtwo\big((Q_J)_\#\mu,\,(Q_J)_\#\nu\big).
\]
\end{lemma}

\begin{proof}
Push any coupling through $(Q_J,Q_J)$ and use $\|Q_J(u)-Q_J(v)\|\le \|u-v\|$.
\end{proof}

\section{Subgrid forcing at a cutoff: what memory should predict}
\label{app:subgrid}

This section makes the cutoff viewpoint concrete by writing down an explicit ``resolved equation + remainder'' identity.
Even if the generator predicts the next resolved state directly (rather than predicting a forcing term), this computation explains
what statistical information about the unresolved tail is relevant: it is exactly what controls the subgrid remainder.

For incompressible Navier--Stokes on $\T^d$,
\[
\partial_t u + \mathbb{P}\nabla\cdot(u\otimes u)=\nu\Delta u,\qquad \nabla\cdot u=0,
\]
define $y=S_Ju$ and $z=Q_Ju$. Then the resolved equation becomes
\[
\partial_t y + \mathbb{P}\nabla\cdot S_J(y\otimes y)=\nu\Delta y + \mathcal{R}_J(u),
\]
where
\[
\mathcal{R}_J(u)
:=
-\mathbb{P}\nabla\cdot S_J\big(y\otimes z + z\otimes y + z\otimes z\big).
\]
Even with $y$ fixed, $\mathcal{R}_J$ depends on tail statistics of $z$ and its correlations.

\begin{lemma}[Bernstein]\label{lem:app-bernstein}
There exists $C=C(d)$ such that for all $f\in L^2(\T^d)$,
\[
\|S_J f\|_{L^2}\le \|f\|_{L^2},
\qquad
\|\nabla S_J f\|_{L^2}\le C\,2^J\|f\|_{L^2}.
\]
Moreover, for $1\le p\le q\le\infty$,
\[
\|S_J f\|_{L^q}\le C\,2^{Jd(\frac1p-\frac1q)}\|f\|_{L^p}.
\]
\end{lemma}

\begin{proof}
Standard Bernstein estimates for low-frequency Fourier multipliers.
\end{proof}

\begin{proposition}[Subgrid forcing controlled by unresolved tail]\label{prop:app-RJ-tail}
Let $u\in L^2(\T^d)$ with $y=S_Ju$ and $z=Q_Ju$. Then
\[
\|\mathcal{R}_J(u)\|_{L^2}
\le
C\,2^J\Big(\|y\|_{L^\infty}\,\|z\|_{L^2}+\|z\|_{L^4}^2\Big),
\]
with $C=C(d)$.
\end{proposition}

\begin{proof}
Use $\|\mathbb{P}g\|_2\le\|g\|_2$ and $\|\nabla S_J f\|_2\lesssim 2^J\|f\|_2$, then bound
$\|y\otimes z\|_2\le\|y\|_\infty\|z\|_2$ and $\|z\otimes z\|_2=\|z\|_4^2$.
\end{proof}

\section{Memory reduces irreducible uncertainty and induces structured priors}
\label{app:projection}

This section links the conditioning viewpoint (MMSE reduction) and the kernel viewpoint (irreducible mismatch).
Together they justify why memory is not merely an architectural embellishment: it changes what can be achieved at the level of
conditional laws.

Lemma~\ref{lem:cond-proj} is the formal MMSE monotonicity statement under added conditioning.

\subsection{Concrete tail observables}
\label{app:projection-observables}

In the cutoff setting, take $Y:=S_Ju$ and $Z:=Q_Ju$.
Examples of tail observables $G(Z)$:
(i) shell energies $(\|\Delta_j Z\|_{L^2}^2)_{j>J}$,
(ii) subgrid forcing $\mathcal{R}_J(Y+Z)$.
With a memory $M$ measurable w.r.t.\ a past window of resolved states, Lemma~\ref{lem:cond-proj} yields
\[
\E\|G(Z)-\E[G(Z)\mid Y,M]\|^2 \le \E\|G(Z)-\E[G(Z)\mid Y]\|^2,
\]
i.e.\ memory can strictly reduce irreducible tail uncertainty.

\subsection{Kernel-level obstruction}
\label{app:projection-irreducible}

Proposition~\ref{prop:irreducible} is the Wasserstein obstruction: if the true conditional kernel depends on $m$, any memoryless
kernel must incur a nonzero mismatch for at least one $m$.

\section{Rectified flows: kernel sensitivity and transport budget}
\label{app:rf}

This section connects the abstract ``kernel'' viewpoint used in the rollout bound to an actual rectified-flow generator.
We (i) define the induced conditional transition kernel,
(ii) show how Lipschitz regularity of the vector field and base law implies Lipschitz regularity of the kernel in memory,
and (iii) record the elementary one-step decomposition that separates \emph{memory approximation} from \emph{conditional generation}
at a fixed memory input.

\subsection{Memory-conditioned rectified-flow kernel}
\label{app:rf-kernel-def}

Fix conditioning $c=(y,m,t)$. Let $\nu_0(\cdot\mid y,m,t)\in\mathcal{P}_2(\mathcal{H})$ be a base law.
Let $\tau\in[0,1]$ denote internal rectified-flow time and
\[
v_\theta:[0,1]\times \mathcal{H}\times Y\times\R^{d_m}\times[0,T]\to \mathcal{H}
\]
be the learned vector field. Define $U_\tau$ by
\[
\frac{d}{d\tau}U_\tau=v_\theta(\tau,U_\tau;y,m,t),
\qquad
U_0\sim \nu_0(\cdot\mid y,m,t),
\]
and set $\widehat u:=U_1$. This induces the conditional kernel
\[
\widehat K_\theta(\cdot\mid y,m,t)=(\Phi^{y,m,t}_\theta)_\#\,\nu_0(\cdot\mid y,m,t),
\]
where $\Phi^{y,m,t}_\theta$ is the time-$1$ flow map.

\subsection{Sensitivity of the flow map to memory}
\label{app:rf-sensitivity}

\begin{assumption}[Lipschitz regularity]\label{ass:app-rf-lip}
There exist $L_u,L_m\ge0$ such that for all $\tau\in[0,1]$,
\[
\|v_\theta(\tau,u;y,m,t)-v_\theta(\tau,u';y,m,t)\|\le L_u\|u-u'\|,
\qquad
\|v_\theta(\tau,u;y,m,t)-v_\theta(\tau,u;y,m',t)\|\le L_m\|m-m'\|.
\]
\end{assumption}

\begin{lemma}[Flow-map sensitivity in memory]\label{lem:app-phi-m}
Under Assumption~\ref{ass:app-rf-lip}, for fixed $(y,t)$ and fixed initial state $u_0$,
\[
\|\Phi^{y,m,t}_\theta(u_0)-\Phi^{y,m',t}_\theta(u_0)\|
\le
\Big(\frac{e^{L_u}-1}{L_u}\Big)L_m\,\|m-m'\|,
\]
with $(e^{L_u}-1)/L_u=1$ if $L_u=0$.
\end{lemma}

\begin{proof}
Let $U_\tau,\widetilde U_\tau$ solve the ODE with memories $m,m'$ and same $u_0$.
Then $\frac{d}{d\tau}\|U_\tau-\widetilde U_\tau\|\le L_u\|U_\tau-\widetilde U_\tau\|+L_m\|m-m'\|$.
Integrate after multiplying by $e^{-L_u\tau}$.
\end{proof}

\subsection{Kernel sensitivity in memory}
\label{app:rf-kernel-mem}

\begin{assumption}[Base-law continuity in memory]\label{ass:app-base-lip}
There exists $L_\nu\ge0$ such that for all $y,t$ and $m,m'$,
\[
\Wtwo\big(\nu_0(\cdot\mid y,m,t),\nu_0(\cdot\mid y,m',t)\big)\le L_\nu\|m-m'\|.
\]
\end{assumption}

\begin{proposition}[Kernel Lipschitz bound in memory]\label{prop:app-kernel-lip}
Under Assumptions~\ref{ass:app-rf-lip}--\ref{ass:app-base-lip},
\[
\Wtwo\big(\widehat K_\theta(\cdot\mid y,m,t),\widehat K_\theta(\cdot\mid y,m',t)\big)
\le
\Big[\Big(\frac{e^{L_u}-1}{L_u}\Big)L_m + e^{L_u}L_\nu\Big]\|m-m'\|.
\]
\end{proposition}

\begin{proof}
Triangle inequality:
\[
\Wtwo(\Phi_m\#\nu_m,\Phi_{m'}\#\nu_{m'})
\le
\Wtwo(\Phi_m\#\nu_m,\Phi_{m'}\#\nu_m)
+
\Wtwo(\Phi_{m'}\#\nu_m,\Phi_{m'}\#\nu_{m'}).
\]
Bound the first term by Lemma~\ref{lem:app-phi-m} using a shared $U_0\sim\nu_m$.
Bound the second by Lemma~\ref{lem:app-lip-push} with Lipschitz constant $e^{L_u}$ and Assumption~\ref{ass:app-base-lip}.
\end{proof}

\begin{assumption}[Kernel Lipschitz in memory]\label{ass:lip-mem}
There exists $L_{\mathrm{mem}}\ge0$ such that for all $(y,t)$ and $m,m'$,
\[
\Wtwo\big(\widehat K_\theta(\cdot\mid y,m,t),\widehat K_\theta(\cdot\mid y,m',t)\big)
\le L_{\mathrm{mem}}\|m-m'\|.
\]
\end{assumption}

\paragraph{Remark.}
Assumption~\ref{ass:lip-mem} holds with
$L_{\mathrm{mem}}=\big(\frac{e^{L_u}-1}{L_u}\big)L_m+e^{L_u}L_\nu$ under Proposition~\ref{prop:app-kernel-lip}.
(Compare also Assumption~\ref{ass:lip-mem-fixed}, which is the same condition stated in the local notation of the rollout section.)

\subsection{One-step kernel decomposition}
\label{app:rf-onestep}

Let $K^\star(\cdot\mid y,m,t)$ be the true conditional transition law given $(y,m,t)$.
Define the RF approximation error
\[
\varepsilon_{\mathrm{RF}}(y,m,t)
:=
\Wtwo\big(\widehat K_\theta(\cdot\mid y,m,t),\,K^\star(\cdot\mid y,m,t)\big).
\]
Then for any candidate memory input $\widehat m$,
\begin{equation}\label{eq:app-one-step-kernel}
\Wtwo\big(\widehat K_\theta(\cdot\mid y,\widehat m,t),\,K^\star(\cdot\mid y,m,t)\big)
\le
L_{\mathrm{mem}}\|\widehat m-m\| + \varepsilon_{\mathrm{RF}}(y,m,t),
\end{equation}
by the triangle inequality and Assumption~\ref{ass:lip-mem}.

\paragraph{Interpretation.}
Inequality \eqref{eq:app-one-step-kernel} is the clean separation used throughout the paper:
\emph{memory approximation} ($\|\widehat m-m\|$) and \emph{conditional generation} ($\varepsilon_{\mathrm{RF}}$) enter additively.
This is exactly why it is meaningful to talk about ``memory helps more under OOD schedules'': if $\widehat m$ drifts under compounding
errors, the first term grows, and the sensitivity constant $L_{\mathrm{mem}}$ quantifies how much that drift matters at the level
of predicted laws.

\subsection{Transport budget and ``RF works harder''}
\label{app:rf-budget}

If $(\rho_\tau)_{\tau\in[0,1]}$ satisfies $\partial_\tau\rho_\tau+\nabla\cdot(b_\tau\rho_\tau)=0$ on $\mathcal{H}$, then
\[
\Wtwo^2(\rho_0,\rho_1)\le \int_0^1\int_{\mathcal{H}}\|b_\tau(u)\|^2\,d\rho_\tau(u)\,d\tau
\]
(Benamou--Brenier, Theorem in 4.1.5~\cite{figalli2021invitation}). For deterministic RF, $b_\tau=v_\theta(\tau,\cdot;c)$ and $\rho_\tau=\Law(U_\tau)$.
Combined with Lemma~\ref{lem:W2-reduction}, if low modes are pinned and only the tail is randomized, then the $\Wtwo$ mismatch (hence the
minimal transport energy) is exactly the mismatch of the tail laws.

\paragraph{Connection back to Theorem~\ref{thm:tau-shell-memory}.}
When internal-time denoising is prior-driven in many shells, the burden shifts to providing a good prior mean/structure for the tail.
Memory provides precisely such structure (through a learned summary of the past resolved trajectory), reducing the transport needed to
populate missing frequencies.

\section{Conditional-to-law lifting: from conditional mismatch to $\epsilon_n$}
\label{app:conditional}

This appendix provides the disintegration/coupling step that turns conditional $\Wtwo$ bounds into the law-level one-step defect
$\epsilon_n$ in \eqref{eq:theory-epsn}. Concretely: the rollout bound needs $\epsilon_n$, whereas the model is trained and analyzed
in terms of conditional transition laws (kernels). The bridge between the two is that the rollout law is a \emph{mixture} over the
conditioning variable, and mixtures can be coupled componentwise.

Let $\widehat\mu_n$ be the model law at time $t_n$ and define a conditioning variable
$C_n := \Gamma_n(\widehat u_n)$ in some Polish space $\mathsf{C}$.
Let $\widehat\lambda_n:=\Law(C_n)$ and disintegrate:
\[
\widehat\mu_n(du)=\int \widehat\mu_n(du\mid c)\,\widehat\lambda_n(dc).
\]
Define the true conditional next-step law under the current rollout distribution by
\[
K_n^\star(\cdot\mid c):=\Law\big(S_{\Delta t_n}(\widehat u_n)\,\big|\,C_n=c\big),
\]
so that $(S_{\Delta t_n})_\#\widehat\mu_n=\int K_n^\star(\cdot\mid c)\,d\widehat\lambda_n(c)$.
The model step defines $c\mapsto \widehat K_{\theta,n}(\cdot\mid c)$ so that
$\widehat\mu_{n+1}=\int \widehat K_{\theta,n}(\cdot\mid c)\,d\widehat\lambda_n(c)$.
Therefore, by Lemma~\ref{lem:mixture-common},
\begin{equation}\label{eq:app-eps-rms}
\epsilon_n^2
=
\Wtwo^2\!\big((S_{\Delta t_n})_\#\widehat\mu_n,\widehat\mu_{n+1}\big)
\le
\int \Wtwo^2\big(K_n^\star(\cdot\mid c),\widehat K_{\theta,n}(\cdot\mid c)\big)\,d\widehat\lambda_n(c).
\end{equation}

\paragraph{Summary.}
Equation \eqref{eq:app-eps-rms} is the point where the ``kernel view'' meets the ``rollout view'':
a bound on conditional mismatch (inside the integral) immediately yields a bound on the law-level defect $\epsilon_n$.
Plugging that into Lemma~\ref{lem:recursion} and Theorem~\ref{thm:rollout} yields the final long-horizon control.
In particular, if the conditional mismatch is decomposed as in \eqref{eq:app-one-step-kernel}, then \eqref{eq:app-eps-rms} yields a
rollout-relevant bound that separates memory approximation from conditional generation, averaged over the \emph{rollout} distribution
of conditioning variables.

\section{Model and implementation details}
\label{app:impl}

This appendix summarizes the model components and training protocol used in all experiments, at a level sufficient to reproduce the
implementation. Unless stated otherwise, the memory-conditioned model and its memoryless baseline share the same backbone,
parameter count up to the memory module, and optimization settings. Throughout, $u_n\in\R^{C\times H\times W}$ denotes the state at
physical time step $n$, $\Delta t_n$ the physical time increment, and $\tau\in[0,1]$ the internal rectified-flow time.

\subsection{Rectified-flow parameterization and conditioning interface}
\label{app:impl-backbone}

\paragraph{Rectified-flow ODE and learned vector field.}
For each physical step $n$, the model defines a conditional distribution for $u_{n+1}$ by sampling the rectified-flow ODE
\begin{equation}
\label{eq:rf_ode_app}
\frac{d}{d\tau} u_\tau \;=\; v_\theta\!\big(\tau,u_\tau;\,\tilde u_n, m_n, \Delta t_n\big),
\qquad
u_{\tau=0}\sim \mathcal N(0,I),
\qquad
\widehat u_{n+1}:=u_{\tau=1},
\end{equation}
where $\tilde u_n$ is the conditioning state provided to the model at step $n$ (ground truth under teacher forcing, otherwise a
model prediction), and $m_n$ is the memory context. The memoryless baseline corresponds to $m_n\equiv 0$.

\paragraph{Time embeddings.}
Both the internal time $\tau$ and the physical time step $\Delta t_n$ are embedded with sinusoidal features and injected via
feature-wise modulation (FiLM/AdaLN-style) into each residual block. Concretely, the embeddings are mapped through learned MLPs to
produce scale/shift parameters applied to normalized activations.

\paragraph{Backbone architecture.}
The vector field $v_\theta$ is parameterized by a U-Net-style multi-resolution convolutional network \cite{saharia2022photorealistic}:
\begin{itemize}[leftmargin=*,itemsep=2pt]
\item \textbf{Base width:} $d=64$ with channel multipliers $(1,2,4)$ (feature widths $[64,128,256]$).
\item \textbf{Down/up-sampling:} factor-$2$ pooling/upsampling between levels (three resolutions).
\item \textbf{Attention:} self-attention at the bottleneck only (4 heads, head dimension 32).
\item \textbf{Blocks:} residual blocks with two $3\times 3$ convolutions per block.
\item \textbf{Nonlinearity:} GELU.
\item \textbf{Normalization:} GroupNorm (8 groups) before each convolution.
\item \textbf{Dropout:} $0.1$.
\end{itemize}

\paragraph{Conditioning interface.}
At each resolution level, we concatenate the noised state $u_\tau$ and the conditioning state $\tilde u_n$ channel-wise and feed the
result to the corresponding U-Net block. For CE-RM ($C=4$ channels), this yields an 8-channel input (4 from $u_\tau$ and 4 from
$\tilde u_n$). Conditioning is thus available throughout the multi-resolution path; memory is injected only at the bottleneck (below).

\subsection{Memory module: embedding, temporal aggregation, and injection}
\label{app:impl-memory}

The memory mechanism maintains a compact temporal context $m_n\in\R^{d_m}$ with $d_m=128$.
Operationally, it is built from (i) a per-step embedding $e_n$ extracted from the conditioning input $\tilde u_n$ and (ii) a short
history of past embeddings processed by an S4 state-space layer \cite{gu2021efficiently}.

\paragraph{Embedding.}
At physical step $n$ we compute
\begin{equation}
\label{eq:mem_embed_app}
e_n \;=\; \mathsf{E}_\psi(\tilde u_n,\Delta t_n)\in\R^{d_m},
\end{equation}
where $\mathsf{E}_\psi$ is a convolutional encoder producing bottleneck features, followed by global average pooling and a linear
projection to $\R^{128}$. The scalar $\Delta t_n$ is embedded sinusoidally and concatenated before the final projection.

\paragraph{Temporal aggregation (finite history).}
We maintain a FIFO buffer of the most recent $L_{\max}=16$ embeddings,
\[
\mathcal B_n := (e_{n-\ell+1},\ldots,e_n)\in(\R^{d_m})^\ell, \qquad \ell=\min\{n+1,L_{\max}\}.
\]
To obtain a memory context we apply an S4 layer to the padded/truncated sequence $\mathcal B_n$ and read out the last position:
\begin{equation}
\label{eq:mem_context_app}
m_n \;=\; \mathsf{Readout}\Big(\mathsf{S4}(\mathsf{PadTo}_{L_{\max}}(\mathcal B_n))\Big)\in\R^{d_m}.
\end{equation}
We use a standard S4 configuration (HiPPO-style initialization and bilinear discretization), followed by LayerNorm and a learned
projection back to $\R^{d_m}$. Because $L_{\max}$ is small, recomputing the S4 outputs over the buffer is inexpensive and keeps the
implementation simple and deterministic.

\paragraph{Injection into the vector field.}
Memory is injected at the U-Net bottleneck via gated fusion. Let $h\in\R^{B\times d_h\times H'\times W'}$ denote the bottleneck
features ($d_h=256$) and $m\in\R^{B\times d_m}$ the memory context.
We first broadcast $m$ spatially and project to $d_h$ channels:
\[
m_{\mathrm{sp}} := \mathrm{Repeat}(m[:,:,None,None])\in\R^{B\times d_m\times H'\times W'},
\qquad
m_{\mathrm{proj}} := \mathrm{Conv}_{1\times 1}(m_{\mathrm{sp}})\in\R^{B\times d_h\times H'\times W'}.
\]
We then compute gates $g_h,g_m\in(0,1)^{B\times d_h\times H'\times W'}$ and fuse:
\begin{equation}
\label{eq:gated_fusion_app}
g_h := \sigma(\mathrm{Conv}_{1\times 1}(h)),\qquad
g_m := \sigma(\mathrm{Conv}_{1\times 1}(m_{\mathrm{proj}})),\qquad
\widetilde h := h\odot g_h \;+\; m_{\mathrm{proj}}\odot g_m.
\end{equation}
This exposes the vector field to temporal context at every internal-time evaluation $\tau$ during sampling, while keeping the memory
update frequency aligned with the physical time steps (see below).

\subsection{Base law}
\label{app:impl-baselaw}

The base distribution for flow matching is a fixed isotropic Gaussian on the data tensor space:
\[
\nu_0(\cdot \mid \tilde u_n,m_n,\Delta t_n)\equiv \mathcal N(0,I)\quad \text{on }\R^{C\times H\times W}.
\]
We also experimented with memory-conditioned Gaussian priors, but did not observe consistent improvements on our benchmarks; all
results therefore use the fixed base law.

\subsection{Training objectives and data regimes}
\label{app:impl-training}

Training mixes (i) single-step pairs and (ii) short autoregressive windows. The memoryless baseline uses the same batches and losses,
with memory disabled.

\paragraph{Single-step (A2A) supervision.}
Given a pair $(u_n,u_{n+1})$ with lead time $\Delta t_n$, we draw $u_0\sim\mathcal N(0,I)$ and $\tau\sim \mathrm{Unif}[0,1]$ and form
the linear interpolant
\[
u_\tau := (1-\tau)u_0 + \tau u_{n+1}.
\]
The rectified target velocity is $v^\star := u_{n+1}-u_0$, and we minimize
\begin{equation}
\label{eq:fm_loss_app}
\mathcal L_{\mathrm{FM}}
:=\E\Big[\big\|v_\theta(\tau,u_\tau;\,u_n,m,\Delta t_n) - (u_{n+1}-u_0)\big\|_2^2\Big],
\end{equation}
where $m=0$ in A2A mode (and always in the memoryless baseline).

\paragraph{Sequence supervision and teacher forcing.}
For sequence windows of length $3$--$12$, we unroll autoregressively to expose the model to compounding errors. At each physical step
within the window we choose the conditioning state $\tilde u_n$ by teacher forcing: with probability $p_{\mathrm{TF}}$ we set
$\tilde u_n=u_n$ (ground truth), otherwise $\tilde u_n=\widehat u_n$ (model prediction). We then compute $(m_n,e_n)$ from $\tilde u_n$
and sample $\widehat u_{n+1}$ with a cheap internal-time integrator (Algorithm~\ref{alg:sequence-training}).
The per-step loss is still a flow-matching loss against the \emph{ground truth} next state $u_{n+1}$.

\paragraph{Degraded-conditioning (evaluation-time) variant.}
In a designated set of experiments, we evaluate robustness under degraded conditioning by replacing the \emph{conditioning input}
$\tilde u_n$ with a downsample--upsample version
\[
\tilde u_n^{\mathrm{deg}} := \mathrm{Upsample}\!\big(\mathrm{AvgPool}_{2\times 2}(\tilde u_n)\big),
\]
while keeping targets $u_{n+1}$ at full resolution. We do \emph{not} retrain a separate model for this variant unless explicitly
stated; it is a test-time modification of the conditioning channel.

\subsection{Optimization and sampling}
\label{app:impl-optim}

\paragraph{Optimization.}
We use Adam with gradient clipping ($\|\nabla\|_2\le 1$) and a cosine-annealing learning-rate schedule with warm restarts.
Sequence batches backpropagate through the autoregressive chain (truncated BPTT). To prevent unstable long-range gradients through
memory, we store \emph{detached} embeddings in the FIFO buffer during sequence training (Algorithm~\ref{alg:sequence-training}).

\paragraph{Sampling and update timing.}
At evaluation, sampling integrates the rectified-flow ODE \eqref{eq:rf_ode_app} from $\tau=0$ to $\tau=1$ using either Euler or Heun
(RK2) with $K$ internal-time steps per physical step.
The memory buffer is updated \emph{once per physical step}, using the embedding extracted from the (possibly degraded) conditioning
state $\tilde u_n$; it is \emph{not} updated on internal-time substeps.

\subsection{Hyperparameters}
\label{app:impl-hparams}

\begin{table}[t]
\centering
\caption{Model and training hyperparameters.}
\label{tab:app-hparams}
\small
\setlength{\tabcolsep}{4pt}
\renewcommand{\arraystretch}{0.95}
\begin{tabular}{@{}ll@{}}
\toprule
\textbf{Parameter} & \textbf{Value} \\
\midrule
\textit{Architecture} \\
Backbone & U-Net with residual blocks \\
Base dimension & $d=64$ \\
Channel multipliers & $(1,2,4)$ \\
Attention & bottleneck only (4 heads, head dim 32) \\
Dropout & $0.1$ \\
Activation / Norm & GELU / GroupNorm (8 groups) \\
\midrule
\textit{Memory} \\
Memory dimension & $d_m=128$ \\
S4 state dimension & $n=64$ \\
Max history length & $L_{\max}=16$ \\
Injection & gated fusion at bottleneck \\
\midrule
\textit{Training and sampling} \\
Teacher forcing probability & $p_{\mathrm{TF}}=0.8$ \\
Sequence length range & $3$--$12$ steps \\
Base law & $\mathcal{N}(0,I)$ \\
Internal-time integrator & Euler or Heun (RK2) \\
\bottomrule
\end{tabular}
\end{table}
\FloatBarrier

\section{Experimental details and additional results}
\label{app:exp}

This appendix records dataset construction, metrics, and extended tables supporting the main paper claims.

\subsection{Datasets and data generation}
\label{app:exp-data}

We evaluate on two 2D compressible Euler benchmarks from the \textsc{PDEgym} collection on the Hugging Face Hub
\cite{pdegym_hf}. Concretely, we use \texttt{CE-CRP} (multiple curved Riemann problems) and \texttt{CE-RM}
(Richtmyer--Meshkov), both simulated on the unit square at $128\times 128$ resolution and stored as $21$ uniformly spaced
snapshots in time \cite{pdegym_hf}. Each trajectory provides five physical fields; \texttt{CE-CRP} includes
(density, two velocity components, pressure, energy), while \texttt{CE-RM} includes (density, two velocity components,
pressure, passive tracer) \cite{pdegym_hf}. The datasets share the same underlying numerical pipeline (solver stack and
discretization choices), but differ in their randomized initial-condition families and thus in the resulting multiscale
structure (shock/contact interactions for \texttt{CE-CRP} versus instability-driven mixing with a tracer for \texttt{CE-RM}).
The \textsc{PDEgym} dataset cards specify the simulation parameters and train/val/test splits (e.g.\ $9640/120/240$ for
\texttt{CE-CRP} and $1030/100/130$ for \texttt{CE-RM}) \cite{pdegym_hf}, while full details of the solver stack and dataset
generation are described in Appendix~\textbf{B} of \cite{herde2024poseidon}.

\paragraph{CE-RM (Richtmyer--Meshkov mixing).}
Initial conditions contain perturbed interfaces subjected to impulsive acceleration, producing shock--interface interactions and
rapid fine-scale mixing.

\paragraph{CRP2D (four-quadrant Riemann problems).}
Initial conditions are piecewise-constant four-quadrant states producing interacting shocks, contacts, and rarefactions.

\subsection{Algorithms: training and inference}
\label{app:impl-training-detailed}

We now give explicit pseudocode for the components described above. The intent is to pin down (i) how memory is constructed from a
finite embedding history, (ii) how teacher forcing is applied, and (iii) the precise timing of memory updates (once per physical
step).

\subsubsection{Memory buffer (FIFO) and S4-based context} \label{app:memory-algorithms}

We store the last $L_{\max}$ embeddings in a FIFO buffer. In sequence training, embeddings are detached before insertion to prevent
unbounded gradient paths through the buffer.

\begin{algorithm}[h]
\caption{EmbeddingBuffer: FIFO buffer of the last $L_{\max}$ embeddings}
\label{alg:memory-manager}
\begin{algorithmic}[1]
\STATE \textbf{Initialize:} buffer $\mathcal B\gets [\,]$, maximum length $L_{\max}$
\STATE
\STATE \textbf{procedure} \textsc{Reset}()
\STATE \quad $\mathcal B\gets[\,]$
\STATE \textbf{end procedure}
\STATE
\STATE \textbf{procedure} \textsc{Push}($e\in\R^{d_m}$)
\STATE \quad Append $e$ to the end of $\mathcal B$
\IF{$|\mathcal B| > L_{\max}$}
    \STATE \quad Remove the first (oldest) element of $\mathcal B$
\ENDIF
\STATE \textbf{end procedure}
\STATE
\STATE \textbf{procedure} \textsc{Get}()
\RETURN $\mathcal B$ \COMMENT{a list of length $\ell\le L_{\max}$}
\STATE \textbf{end procedure}
\end{algorithmic}
\end{algorithm}

\begin{algorithm}[h]
\caption{Memory context from conditioning state (embedding + S4 + readout)}
\label{alg:s4-memory}
\begin{algorithmic}[1]
\STATE \textbf{Input:} conditioning state $\tilde u_n\in\R^{C\times H\times W}$, step size $\Delta t_n$, buffer $\mathcal B$ (list of past embeddings)
\STATE \textbf{Output:} memory context $m_n\in\R^{d_m}$, current embedding $e_n\in\R^{d_m}$

\STATE $e_n \gets \mathsf{E}_\psi(\tilde u_n,\Delta t_n)$ \COMMENT{Eq.~\eqref{eq:mem_embed_app}}
\STATE $\mathcal B' \gets \mathcal B$ with $e_n$ appended at the end
\STATE Truncate $\mathcal B'$ to its last $L_{\max}$ elements if needed
\STATE $S \gets \mathsf{PadTo}_{L_{\max}}(\mathcal B')$ \COMMENT{$S\in\R^{L_{\max}\times d_m}$, zero-pad on the left}
\STATE $Y \gets \mathsf{S4}(S)$ \COMMENT{$Y\in\R^{L_{\max}\times d_m}$}
\STATE $m_n \gets \mathsf{Readout}(Y[L_{\max},:])$ \COMMENT{LayerNorm + Linear to $\R^{d_m}$}
\RETURN $m_n, e_n$
\end{algorithmic}
\end{algorithm}

\paragraph{Remark on indexing.}
In Algorithm~\ref{alg:s4-memory}, $m_n$ is computed from the history \emph{including} the current embedding $e_n$ extracted from
$\tilde u_n$. This is the most natural choice in our setting: memory summarizes ``everything known at step $n$'' and is then used to
sample the next state.

\subsubsection{Gated bottleneck fusion (implementation-level)}
\label{app:gated-fusion}

\begin{algorithm}[h]
\caption{Gated fusion at the bottleneck}
\label{alg:gated-fusion}
\begin{algorithmic}[1]
\STATE \textbf{Input:} bottleneck features $h\in\R^{B\times d_h\times H'\times W'}$ ($d_h=256$), memory $m\in\R^{B\times d_m}$ ($d_m=128$)
\STATE \textbf{Output:} fused features $\widetilde h\in\R^{B\times d_h\times H'\times W'}$

\STATE $m_{\mathrm{sp}} \gets \mathrm{Repeat}(m[:,:,None,None])$ to shape $B\times d_m\times H'\times W'$
\STATE $m_{\mathrm{proj}} \gets \mathrm{Conv}_{1\times 1}(m_{\mathrm{sp}})$ \COMMENT{$d_m\to d_h$}
\STATE $g_h \gets \sigma(\mathrm{Conv}_{1\times 1}(h))$
\STATE $g_m \gets \sigma(\mathrm{Conv}_{1\times 1}(m_{\mathrm{proj}}))$
\STATE $\widetilde h \gets h\odot g_h + m_{\mathrm{proj}}\odot g_m$
\RETURN $\widetilde h$
\end{algorithmic}
\end{algorithm}

\textbf{Interpretation.}
The gates allow the network to modulate how strongly temporal context influences the bottleneck representation. When memory is
uninformative, the model can down-weight $m_{\mathrm{proj}}$ via $g_m$ and revert to a purely spatial representation.

\subsubsection{Training: mixing single-step and sequence batches}
\label{app:mixed-training}

Training alternates between (i) single-step A2A batches (pure flow matching, no memory) and (ii) sequence batches (teacher forcing +
memory). This mixture prevents the sequence regime from dominating optimization while still exposing the model to compounding-error
scenarios.

\begin{algorithm}[h]
\caption{Mixed training loop (A2A + sequences)}
\label{alg:mixed-training}
\begin{algorithmic}[1]
\STATE \textbf{Hyperparameters:} $p_{\mathrm{seq}}$ (prob.\ of sequence batch), $p_{\mathrm{TF}}$ (teacher forcing prob.)
\FOR{$\text{iter}=1,\ldots$}
    \IF{$\text{Bernoulli}(p_{\mathrm{seq}})=1$}
        \STATE Sample a batch of short windows $\{(u_0,\ldots,u_L;\Delta t_0,\ldots,\Delta t_{L-1})\}$
        \STATE $\mathcal L \gets \textsc{ForwardSequence}(\text{batch},p_{\mathrm{TF}})$ \COMMENT{Alg.~\ref{alg:sequence-training}}
    \ELSE
        \STATE Sample a batch of pairs $\{(u_n,u_{n+1},\Delta t_n)\}$
        \STATE $\mathcal L \gets \textsc{ForwardA2A}(\text{batch})$ \COMMENT{Alg.~\ref{alg:a2a-training}}
    \ENDIF
    \STATE $\mathcal L.\mathrm{backward}()$
    \STATE Clip gradient norm to $1.0$
    \STATE $\mathrm{optimizer.step}()$;\quad $\mathrm{optimizer.zero\_grad}()$
    \STATE $\mathrm{scheduler.step}()$ \COMMENT{cosine annealing with restarts}
\ENDFOR
\end{algorithmic}
\end{algorithm}

\subsubsection{A2A (single-step) training}
\label{app:a2a-training}

\begin{algorithm}[h]
\caption{\textsc{ForwardA2A}: single-step flow matching}
\label{alg:a2a-training}
\begin{algorithmic}[1]
\STATE \textbf{Input:} batch $\{(u_n^{(i)},u_{n+1}^{(i)},\Delta t_n^{(i)})\}_{i=1}^B$
\STATE \textbf{Output:} loss $\mathcal L_{\mathrm{A2A}}$

\FOR{$i=1,\ldots,B$}
    \STATE Sample $u_0^{(i)}\sim\mathcal N(0,I)$ and $\tau^{(i)}\sim\mathrm{Unif}[0,1]$
    \STATE $u_\tau^{(i)} \gets (1-\tau^{(i)})u_0^{(i)} + \tau^{(i)}u_{n+1}^{(i)}$
    \STATE $v_{\mathrm{pred}}^{(i)} \gets v_\theta(\tau^{(i)},u_\tau^{(i)};\,u_n^{(i)},m=0,\Delta t_n^{(i)})$
    \STATE $v_{\mathrm{tgt}}^{(i)} \gets u_{n+1}^{(i)} - u_0^{(i)}$
\ENDFOR
\STATE $\mathcal L_{\mathrm{A2A}} \gets \frac{1}{B}\sum_{i=1}^B \|v_{\mathrm{pred}}^{(i)}-v_{\mathrm{tgt}}^{(i)}\|_2^2$
\RETURN $\mathcal L_{\mathrm{A2A}}$
\end{algorithmic}
\end{algorithm}

\paragraph{Sampling of pairs.}
If a trajectory has length $T$, it provides $\binom{T}{2}$ valid pairs. We sample pairs uniformly over trajectories and indices,
which yields broad coverage of lead times in the admissible $\Delta t$ range.

\subsubsection{Sequence training with memory and teacher forcing}
\label{app:sequence-training}

In sequence mode, we must (a) compute a memory context from the conditioning state and (b) produce a predicted next state to feed
forward when teacher forcing is off. For efficiency, we use a \emph{cheap} internal-time sampler during training: we integrate
\eqref{eq:rf_ode_app} with a small number $K_{\mathrm{train}}$ of internal steps (often $K_{\mathrm{train}}=1$ with Euler). The loss
at each step remains a flow-matching loss against the true next state $u_{n+1}$.

\begin{algorithm}[h]
\caption{\textsc{ForwardSequence}: autoregressive training with memory (teacher forcing)}
\label{alg:sequence-training}
\begin{algorithmic}[1]
\STATE \textbf{Input:} batch of windows $\{(u_0^{(i)},\ldots,u_{L_i}^{(i)};\Delta t_0^{(i)},\ldots,\Delta t_{L_i-1}^{(i)})\}_{i=1}^B$
\STATE \textbf{Input:} teacher forcing probability $p_{\mathrm{TF}}$, internal steps $K_{\mathrm{train}}$
\STATE \textbf{Output:} sequence loss $\mathcal L_{\mathrm{seq}}$

\FOR{$i=1,\ldots,B$}
    \STATE Initialize buffer $\mathcal B^{(i)}\gets[\,]$ \COMMENT{Alg.~\ref{alg:memory-manager}}
    \STATE Set $\widehat u_0^{(i)}\gets u_0^{(i)}$
\ENDFOR
\STATE $\mathcal S \gets [\,]$ \COMMENT{store per-step losses}

\FOR{$n=0,\ldots,\max_i(L_i)-1$}
    \FOR{each $i$ with $n<L_i$}
        \STATE \textbf{Teacher forcing:}
        \IF{$n=0$ \textbf{or} $\mathrm{Bernoulli}(p_{\mathrm{TF}})=1$}
            \STATE $\tilde u_n^{(i)} \gets u_n^{(i)}$
        \ELSE
            \STATE $\tilde u_n^{(i)} \gets \widehat u_n^{(i)}$
        \ENDIF

        \STATE \textbf{Memory:} $(m_n^{(i)},e_n^{(i)}) \gets \textsc{MemoryFromState}(\tilde u_n^{(i)},\Delta t_n^{(i)},\mathcal B^{(i)})$
        \STATE \hspace{2.0em}\COMMENT{Alg.~\ref{alg:s4-memory}}

        \STATE \textbf{Training-time sampler (cheap):}
        \STATE Sample $z^{(i)}\sim\mathcal N(0,I)$ and set $u^{(i)} \gets z^{(i)}$
        \STATE Set grid $\tau_k := k/K_{\mathrm{train}}$ for $k=0,\ldots,K_{\mathrm{train}}$
        \FOR{$k=0,\ldots,K_{\mathrm{train}}-1$}
            \STATE $\Delta\tau \gets \tau_{k+1}-\tau_k$
            \STATE $v \gets v_\theta(\tau_k,u^{(i)};\,\tilde u_n^{(i)},m_n^{(i)},\Delta t_n^{(i)})$
            \STATE $u^{(i)} \gets u^{(i)} + \Delta\tau\, v$ \COMMENT{Euler step}
        \ENDFOR
        \STATE $\widehat u_{n+1}^{(i)} \gets u^{(i)}$

        \STATE \textbf{Per-step flow-matching loss:}
        \STATE Sample $\tau\sim\mathrm{Unif}[0,1]$ and form $u_\tau = (1-\tau)z^{(i)} + \tau u_{n+1}^{(i)}$
        \STATE $v_{\mathrm{pred}} \gets v_\theta(\tau,u_\tau;\,\tilde u_n^{(i)},m_n^{(i)},\Delta t_n^{(i)})$
        \STATE $v_{\mathrm{tgt}} \gets u_{n+1}^{(i)} - z^{(i)}$
        \STATE Append $\|v_{\mathrm{pred}}-v_{\mathrm{tgt}}\|_2^2$ to $\mathcal S$

        \STATE \textbf{Update buffer (detach in training):}
        \STATE $\mathcal B^{(i)}.\mathrm{Push}(\mathrm{detach}(e_n^{(i)}))$
    \ENDFOR
\ENDFOR

\STATE $\mathcal L_{\mathrm{seq}} \gets \frac{1}{|\mathcal S|}\sum_{\ell\in\mathcal S}\ell$
\RETURN $\mathcal L_{\mathrm{seq}}$
\end{algorithmic}
\end{algorithm}

\textbf{Key design choices.}
\begin{itemize}[leftmargin=*,itemsep=2pt]
\item \textbf{Truncated BPTT through predictions:} gradients flow through the autoregressive chain
$\widehat u_0\to \widehat u_1\to\cdots$ within each training window, so the model is directly optimized for short-horizon stability.
\item \textbf{Detached buffer:} we detach $e_n$ before inserting into the buffer, preventing gradients from traversing an arbitrarily
long history through the memory sequence. This stabilizes optimization while preserving the effect of memory at the forward level.
\item \textbf{Cheap sampler in training:} the internal-time integrator in sequence mode uses $K_{\mathrm{train}}$ steps (often 1).
Evaluation uses a more accurate integrator with more steps (Algorithm~\ref{alg:rollout-inference}).
\end{itemize}

\subsubsection{Inference: autoregressive rollout with memory}
\label{app:rollout}

At test time, we run full autoregressive rollouts. The conditioning state $\tilde u_n$ is always the model's current prediction
(unless explicitly evaluating teacher-forced diagnostics). Memory is updated \emph{once per physical step} using the conditioning
input to that step.

\begin{algorithm}[h]
\caption{Autoregressive rollout with memory}
\label{alg:rollout-inference}
\begin{algorithmic}[1]
\STATE \textbf{Input:} initial state $u_0$, steps $N$, step sizes $\{\Delta t_n\}_{n=0}^{N-1}$
\STATE \textbf{Input:} internal-time steps $K$, integrator $\in\{\mathrm{Euler},\mathrm{Heun}\}$, flag $\mathrm{degcond}\in\{0,1\}$
\STATE \textbf{Output:} trajectory $(u_0,\widehat u_1,\ldots,\widehat u_N)$

\STATE Initialize buffer $\mathcal B\gets[\,]$ \COMMENT{Alg.~\ref{alg:memory-manager}}
\STATE Set $\widehat u_0\gets u_0$ and $\mathrm{traj}\gets[\widehat u_0]$

\FOR{$n=0,\ldots,N-1$}
    \STATE $\tilde u_n \gets \widehat u_n$
    \IF{$\mathrm{degcond}=1$}
        \STATE $\tilde u_n \gets \mathrm{Upsample}(\mathrm{AvgPool}_{2\times2}(\tilde u_n))$
    \ENDIF

    \STATE $(m_n,e_n)\gets \textsc{MemoryFromState}(\tilde u_n,\Delta t_n,\mathcal B)$ \COMMENT{Alg.~\ref{alg:s4-memory}}

    \STATE Sample $u \sim \mathcal N(0,I)$ and set $u^{(0)}\gets u$
    \STATE Set $\tau_k := k/K$ for $k=0,\ldots,K$
    \FOR{$k=0,\ldots,K-1$}
        \STATE $\Delta\tau\gets \tau_{k+1}-\tau_k$
        \IF{integrator $=$ Heun}
            \STATE $v_0 \gets v_\theta(\tau_k,u^{(k)};\tilde u_n,m_n,\Delta t_n)$
            \STATE $u^{\mathrm{pred}} \gets u^{(k)} + \Delta\tau\, v_0$
            \STATE $v_1 \gets v_\theta(\tau_{k+1},u^{\mathrm{pred}};\tilde u_n,m_n,\Delta t_n)$
            \STATE $u^{(k+1)} \gets u^{(k)} + \frac{\Delta\tau}{2}(v_0+v_1)$
        \ELSE
            \STATE $v \gets v_\theta(\tau_k,u^{(k)};\tilde u_n,m_n,\Delta t_n)$
            \STATE $u^{(k+1)} \gets u^{(k)} + \Delta\tau\, v$
        \ENDIF
    \ENDFOR
    \STATE $\widehat u_{n+1}\gets u^{(K)}$
    \STATE Append $\widehat u_{n+1}$ to $\mathrm{traj}$

    \STATE $\mathcal B.\mathrm{Push}(e_n)$ \COMMENT{no detachment needed at inference}
\ENDFOR
\RETURN $\mathrm{traj}$
\end{algorithmic}
\end{algorithm}

\textbf{Critical timing detail.}
The buffer update occurs once per physical step using the embedding extracted from the conditioning state $\tilde u_n$ (possibly
degraded), and not from internal-time states $u^{(k)}$. This ensures memory tracks physical evolution and keeps cost linear in the
rollout length (independent of $K$).

\textbf{Training vs.\ inference (summary).}
\begin{center}
\small
\begin{tabular}{lcc}
\toprule
\textbf{Aspect} & \textbf{Training (sequences)} & \textbf{Inference (rollout)} \\
\midrule
Internal-time steps & $K_{\mathrm{train}}$ (small; often 1) & $K$ (e.g.\ 2--8) \\
Memory update & once per physical step & once per physical step \\
Conditioning & teacher forcing w.p.\ $p_{\mathrm{TF}}$ & always autoregressive \\
Gradients & truncated BPTT & none \\
Window/rollout length & 3--12 & arbitrary (10--20+) \\
\bottomrule
\end{tabular}
\end{center}

\subsection{Evaluation metrics}
\label{app:exp-metrics}

\paragraph{Relative $\ell_2$ error.}
We report
\[
\mathrm{rel}\text{-}\ell_2(\widehat u,u)
=\frac{\|\widehat u-u\|_{L^2_x}}{\|u\|_{L^2_x}+\varepsilon},
\qquad \varepsilon=10^{-8},
\]
computed over all channels (and optionally per-channel).

\paragraph{Spectral error decomposition.}
We compute the 2D FFT per variable, radially bin frequencies, and report relative errors in low/mid/high bands. This localizes
whether improvements correspond primarily to large-scale organization or to fine-scale recovery.

\subsection{Baselines and ablations}
\label{app:exp-baselines}

\paragraph{Primary baseline.}
The memoryless baseline uses the identical backbone and training protocol but disables memory ($m_n\equiv 0$). This isolates the
effect of temporal memory under autoregressive deployment.

\paragraph{Ablations (qualitative summary).}
We tested variants that condition only the base law or only the vector field; in this setting, injecting memory into the vector field
accounts for most of the observed gains.

\subsection{Additional tables}
\label{app:exp-tables}

\paragraph{How to read these tables.}
The single-step diagnostics in \Cref{tab:single-step-baseline} verify that memory does not materially change supervised one-step
accuracy, while the spectral analysis in \Cref{tab:spectral-analysis-compact} localizes how rollout improvements distribute across
frequency bands.

\begin{table}[t]
\centering
\caption{%
\textbf{Single-step baseline performance.}
Relative $\ell_2$ errors (mean $\pm$ std) for one-step prediction at short lead times.
Memory has negligible effect in this supervised regime, isolating its benefit to long-horizon rollout stability.}
\label{tab:single-step-baseline}

\scriptsize
\setlength{\tabcolsep}{3pt}
\renewcommand{\arraystretch}{0.95}

\begin{tabular}{@{}lccc@{}}
\toprule
\textbf{Dataset / Training} & $\Delta t$ & \textbf{w/ Mem} & \textbf{w/o Mem} \\
\midrule
\textit{CE-RM / Regular} & 0.05 & $0.079 \pm 0.023$ & $0.078 \pm 0.023$ \\
                        & 0.10 & $0.116 \pm 0.030$ & $0.116 \pm 0.031$ \\
                        & 0.15 & $0.145 \pm 0.044$ & $0.146 \pm 0.040$ \\
\addlinespace[0.25em]
\textit{CE-RM / Downsampled} & 0.05 & $0.130 \pm 0.020$ & $0.130 \pm 0.021$ \\
                            & 0.10 & $0.136 \pm 0.021$ & $0.137 \pm 0.022$ \\
                            & 0.15 & $0.153 \pm 0.034$ & $0.152 \pm 0.032$ \\
\addlinespace[0.25em]
\textit{CRP2D / Regular} & 0.05 & $0.176 \pm 0.011$ & $0.175 \pm 0.010$ \\
                        & 0.10 & $0.244 \pm 0.006$ & $0.244 \pm 0.008$ \\
                        & 0.15 & $0.306 \pm 0.007$ & $0.308 \pm 0.006$ \\
\bottomrule
\end{tabular}

\vspace{0.25em}
\begin{minipage}{\columnwidth}
\scriptsize\emph{Note:} Differences are within $\pm2\%$ at these lead times, consistent with memory targeting long-horizon stability
rather than one-step fidelity.
\end{minipage}
\end{table}


\begin{table}[t]
\centering
\caption{%
\textbf{Spectral error by frequency band (dense-20, CE-RM).}
Each entry reports ``error (Benefit)'' where \textbf{Benefit} is the \% reduction vs.\ the no-memory baseline.}
\label{tab:spectral-analysis-compact}

\scriptsize
\setlength{\tabcolsep}{2.6pt}
\renewcommand{\arraystretch}{0.92}
\begin{tabular}{@{}lccc@{}}
\toprule
\textbf{Var.} & \textbf{Low} & \textbf{Mid} & \textbf{High} \\
\midrule
\multicolumn{4}{@{}l}{\textit{CE-RM (regular)}} \\
$\rho$ & $0.78\ (\mathbf{+29.5\%})$ & $2.10\ (\mathbf{+8.7\%})$ & $9.83\ (\mathbf{+7.3\%})$ \\
$v_x$  & $0.58\ (\mathbf{+40.3\%})$ & $3.31\ (\mathbf{+6.7\%})$ & $29.87\ (-4.2\%)$ \\
$v_y$  & $0.58\ (\mathbf{+42.6\%})$ & $3.46\ (\mathbf{+8.4\%})$ & $30.83\ (+0.2\%)$ \\
$E$    & $0.64\ (\mathbf{+61.4\%})$ & $2.50\ (\mathbf{+9.6\%})$ & $18.69\ (-11.7\%)$ \\
\addlinespace[0.25em]
\multicolumn{4}{@{}l}{\textit{CE-RM (degraded-conditioning)}} \\
$\rho$ & $0.75\ (\mathbf{+30.9\%})$ & $3.68\ (\mathbf{+4.7\%})$ & $8.19\ (\mathbf{+6.6\%})$ \\
$v_x$  & $0.55\ (\mathbf{+37.8\%})$ & $4.07\ (\mathbf{+2.8\%})$ & $20.01\ (-4.4\%)$ \\
$v_y$  & $0.55\ (\mathbf{+41.1\%})$ & $3.81\ (\mathbf{+7.6\%})$ & $18.71\ (+0.9\%)$ \\
$E$    & $0.64\ (\mathbf{+49.4\%})$ & $2.59\ (\mathbf{+22.6\%})$ & $9.20\ (\mathbf{+3.4\%})$ \\
\bottomrule
\end{tabular}
\end{table}


\FloatBarrier


\section{From flow-matching to invariant-measure error control via one-step kernels}
\label{app:mmd_vs_flow}

This appendix records a simple but useful implication chain:
\[
\text{small flow-matching (FM) regression risk}
\ \Longrightarrow\
\text{small one-step conditional law error}
\ \Longrightarrow\
\text{small multi-step / invariant-measure discrepancy (under mixing).}
\]
The point is conceptual as well as technical: invariant-measure regularizers (e.g.\ MMD penalties on long rollouts)
target a fixed point of the Markov operator, whereas conditional generative transition models (rectified flows, diffusion forecasters)
target the operator itself. Under standard mixing assumptions, controlling the operator in a suitable weak metric already controls
the induced invariant measure in the same topology.

\paragraph{Notation.}
We write the conditioning context as
\[
c := (u_n,m,\Delta t)\in\mathcal C,
\]
where $u_n$ is the resolved current state, $m$ is the memory state, and $\Delta t$ is the physical step.
The next-step \emph{true} conditional law is a Markov kernel $K^\star:\mathcal C\times\mathcal B(\mathcal U)\to[0,1]$,
with $K^\star(c,\cdot)\in\mathcal P(\mathcal U)$.
The learned model defines $K_\theta(c,\cdot)$ analogously.

To keep this appendix self-contained, we use the following safe macros (they do not overwrite existing ones).
\providecommand{\U}{\mathcal U}
\providecommand{\Ccal}{\mathcal C}
\providecommand{\B}{\mathcal B}
\providecommand{\Pcal}{\mathcal P}
\providecommand{\R}{\mathbb R}
\providecommand{\E}{\mathbb E}
\providecommand{\Law}{\mathcal L}
\providecommand{\MMD}{\mathrm{MMD}}
\providecommand{\TV}{\mathrm{TV}}
\providecommand{\KL}{\mathrm{KL}}
\providecommand{\Wone}{W_{1}}
\providecommand{\Wtwo}{W_{2}}

\subsection{Markov kernels, operators, and IPM discrepancies}
\label{app:K:kernels-operators}

We work on a measurable state space $(\U,\B(\U))$ and denote by $\Pcal(\U)$ the set of probability measures on it.
A (conditional) Markov kernel is a map $K:\Ccal\times\B(\U)\to[0,1]$ such that $A\mapsto K(c,A)$ is a probability measure
for each $c$, and $c\mapsto K(c,A)$ is measurable for each $A$.

\paragraph{Induced Markov operator.}
Given a distribution $\pi$ on the conditioning contexts $c$ (e.g.\ the distribution of contexts encountered along a rollout),
one obtains an operator on measures on $\U$ by integrating out the context.
To avoid notational clutter, we state the operator form for a single ``current state'' variable $u\in\U$;
the context-dependent case is identical after enlarging the state space to include $(m,\Delta t)$.

Concretely, let $K(u,\cdot)$ be a kernel on $\U$ and define $P:\Pcal(\U)\to\Pcal(\U)$ by
\begin{equation}\label{eq:K:markov-operator}
(P\mu)(A) := \int_{\U} K(u,A)\,\mu(du),\qquad A\in\B(\U).
\end{equation}
The learned model induces $P_\theta$ from $K_\theta$ by the same formula.

\paragraph{Integral probability metrics (IPMs).}
Fix a class $\mathcal G$ of bounded measurable functions $g:\U\to\R$.
Define the IPM
\begin{equation}\label{eq:K:ipm}
D_{\mathcal G}(\mu,\nu)
:= \sup_{g\in\mathcal G}\left|\int_{\U} g\,d\mu-\int_{\U} g\,d\nu\right|.
\end{equation}
The canonical example in this appendix is MMD, where $\mathcal G$ is the RKHS unit ball.

\paragraph{One-step kernel discrepancy.}
Define the worst-case conditional discrepancy (the ``one-step error'')
\begin{equation}\label{eq:K:epsG}
\varepsilon_{\mathcal G} := \sup_{u\in\U} D_{\mathcal G}\bigl(K(u,\cdot),K_\theta(u,\cdot)\bigr).
\end{equation}
(If the kernel depends on context $c$, replace $\sup_{u\in\U}$ with $\sup_{c\in\Ccal}$.)

\begin{lemma}[Operator discrepancy is controlled by the one-step discrepancy]
\label{lem:K:operator-discrepancy}
Assume $u\mapsto D_{\mathcal G}(K(u,\cdot),K_\theta(u,\cdot))$ is measurable.\footnote{This holds under standard separability
assumptions on $\mathcal G$; for MMD it is immediate when the kernel is measurable and bounded.}
Then for every $\mu\in\Pcal(\U)$,
\begin{equation}\label{eq:K:operator-discrepancy}
D_{\mathcal G}(P\mu,P_\theta\mu)
\le \int_{\U} D_{\mathcal G}\bigl(K(u,\cdot),K_\theta(u,\cdot)\bigr)\,\mu(du)
\le \varepsilon_{\mathcal G}.
\end{equation}
\end{lemma}

\begin{proof}
For $g\in\mathcal G$,
\[
\int g\,d(P\mu)-\int g\,d(P_\theta\mu)
=\int_{\U}\Big(\int g\,dK(u,\cdot)-\int g\,dK_\theta(u,\cdot)\Big)\,\mu(du).
\]
Taking absolute values, bounding by the integral of absolute values, and then taking $\sup_{g\in\mathcal G}$ yields the first inequality;
the second is immediate from the definition of $\varepsilon_{\mathcal G}$.
\end{proof}

\subsection{Mixing turns one-step error into multi-step and invariant-measure error}
\label{app:K:mixing}

To keep the bounds explicit, we present a clean contractive (Doeblin-like) regime.
The same message persists under weaker ergodicity hypotheses (Harris recurrence, geometric ergodicity, spectral gaps),
at the expense of more technical constants.

\begin{assumption}[Contractivity in an IPM]
\label{ass:K:contractive}
There exists $\rho\in[0,1)$ such that for all $\mu,\nu\in\Pcal(\U)$,
\begin{equation}\label{eq:K:contractive}
D_{\mathcal G}(P\mu,P\nu)\le \rho\,D_{\mathcal G}(\mu,\nu).
\end{equation}
\end{assumption}

\begin{theorem}[Finite-horizon perturbation bound]
\label{thm:K:finite-horizon}
Under Assumption~\ref{ass:K:contractive} and the measurability hypothesis of Lemma~\ref{lem:K:operator-discrepancy},
for any initial law $\mu_0\in\Pcal(\U)$ and all integers $n\ge 1$,
\begin{equation}\label{eq:K:finite-horizon}
D_{\mathcal G}\bigl(P^n\mu_0,\ P_\theta^n\mu_0\bigr)
\le \frac{1-\rho^n}{1-\rho}\,\varepsilon_{\mathcal G}.
\end{equation}
\end{theorem}

\begin{proof}
Let $\Delta_n := D_{\mathcal G}(P^n\mu_0,P_\theta^n\mu_0)$.
Then $\Delta_0=0$ and
\[
\Delta_{n+1}
= D_{\mathcal G}\bigl(P(P^n\mu_0),P_\theta(P_\theta^n\mu_0)\bigr)
\le D_{\mathcal G}\bigl(P(P^n\mu_0),P(P_\theta^n\mu_0)\bigr)
     + D_{\mathcal G}\bigl(P(P_\theta^n\mu_0),P_\theta(P_\theta^n\mu_0)\bigr).
\]
The first term is $\le \rho\,\Delta_n$ by contractivity, and the second is $\le \varepsilon_{\mathcal G}$ by
Lemma~\ref{lem:K:operator-discrepancy} with $\mu=P_\theta^n\mu_0$. Hence $\Delta_{n+1}\le \rho\Delta_n+\varepsilon_{\mathcal G}$,
so $\Delta_n\le \varepsilon_{\mathcal G}\sum_{j=0}^{n-1}\rho^j = \frac{1-\rho^n}{1-\rho}\varepsilon_{\mathcal G}$.
\end{proof}

\begin{corollary}[Invariant-measure discrepancy bound]
\label{cor:K:invariant}
Assume in addition that $P$ and $P_\theta$ admit invariant measures $\mu^\star$ and $\mu_\theta^\star$
(i.e.\ $P\mu^\star=\mu^\star$ and $P_\theta\mu_\theta^\star=\mu_\theta^\star$).
Then
\begin{equation}\label{eq:K:invariant}
D_{\mathcal G}(\mu^\star,\mu_\theta^\star)\le \frac{1}{1-\rho}\,\varepsilon_{\mathcal G}.
\end{equation}
\end{corollary}

\begin{proof}
Using invariance and the triangle inequality,
\[
D_{\mathcal G}(\mu^\star,\mu_\theta^\star)
= D_{\mathcal G}(P\mu^\star,P_\theta\mu_\theta^\star)
\le D_{\mathcal G}(P\mu^\star,P\mu_\theta^\star)+D_{\mathcal G}(P\mu_\theta^\star,P_\theta\mu_\theta^\star)
\le \rho\,D_{\mathcal G}(\mu^\star,\mu_\theta^\star)+\varepsilon_{\mathcal G},
\]
and rearrange.
\end{proof}

\begin{remark}[Interpretation for ``MMD invariant-measure regularization'']
If $D_{\mathcal G}=\MMD_k$, then \eqref{eq:K:invariant} states that, in a mixing regime,
controlling the \emph{one-step conditional MMD} controls the stationary MMD.
In particular, a transition model trained to match the one-step conditional law can be seen as optimizing a stronger target:
it controls the operator (kernel), which then controls the fixed point (invariant measure) as a consequence.
\end{remark}

\subsection{From rectified-flow / flow-matching risk to one-step kernel error}
\label{app:K:fm-to-kernel}

This section provides the front-end implication needed to connect the training objective of rectified flows
to the kernel discrepancy $\varepsilon_{\mathcal G}$ appearing in Sections~\ref{app:K:kernels-operators}--\ref{app:K:mixing}.

\paragraph{Conditional rectified-flow kernel.}
Fix a context $c\in\Ccal$. The learned model generates a random next state by sampling
$X_0\sim \nu_0(\cdot\mid c)$ and integrating the ODE
\begin{equation}\label{eq:K:rf-ode}
\dot X_\tau = v_\theta(\tau,X_\tau;\,c),\qquad \tau\in[0,1],
\end{equation}
to obtain $X_1$. This defines the learned kernel $K_\theta(c,\cdot)=\Law(X_1)$.
At the population level, flow matching posits (and identifies) a target vector field $v^\star(\tau,\cdot;\,c)$
whose induced transport recovers the true conditional law $K^\star(c,\cdot)$.

To make the link precise, it is convenient to phrase FM in terms of a reference path law.

\begin{assumption}[Flow-matching reference path and population target]
\label{ass:K:fm-reference}
For each $c\in\Ccal$, there exists a family of measures $(\pi_\tau^c)_{\tau\in[0,1]}\subset\Pcal(\U)$ such that
\[
\pi_0^c=\nu_0(\cdot\mid c),\qquad \pi_1^c=K^\star(c,\cdot),
\]
and a measurable vector field $v^\star(\tau,x;\,c)$ such that $(\pi_\tau^c)$ solves the continuity equation
\begin{equation}\label{eq:K:continuity}
\partial_\tau \pi_\tau^c + \nabla\cdot\bigl(v^\star(\tau,\cdot;\,c)\,\pi_\tau^c\bigr)=0
\quad\text{in the sense of distributions on }(0,1)\times\U.
\end{equation}
Moreover, the flow-matching population risk is measured along this reference path:
\begin{equation}\label{eq:K:fm-risk}
\mathcal L_{\mathrm{FM}}(c)
:= \int_0^1 \E_{Z\sim \pi_\tau^c}\bigl\|v_\theta(\tau,Z;\,c)-v^\star(\tau,Z;\,c)\bigr\|^2\,d\tau.
\end{equation}
\end{assumption}

\begin{remark}
Assumption~\ref{ass:K:fm-reference} is the standard ``population'' viewpoint on flow matching:
$\pi_\tau^c$ is the marginal of the interpolation/coupling used in training (for rectified flow, typically linear interpolation
between base noise and data), and $v^\star$ is the population regression target (conditional expectation of the velocity under that coupling).
\end{remark}

\begin{assumption}[Uniform Lipschitz well-posedness]
\label{ass:K:lipschitz}
For each fixed $c$, the maps $x\mapsto v_\theta(\tau,x;\,c)$ and $x\mapsto v^\star(\tau,x;\,c)$ are globally Lipschitz with a constant $L_c$,
uniformly in $\tau\in[0,1]$, and have at most linear growth so that the ODEs are well posed and moments are finite.
\end{assumption}

\begin{theorem}[FM risk controls conditional $\Wtwo$ error]
\label{thm:K:fm-to-w2}
Fix a context $c\in\Ccal$ and assume Assumptions~\ref{ass:K:fm-reference}--\ref{ass:K:lipschitz}.
Let $\Phi_\tau^\theta(\cdot;\,c)$ and $\Phi_\tau^\star(\cdot;\,c)$ be the flows of \eqref{eq:K:rf-ode} with drifts
$v_\theta$ and $v^\star$, respectively, initialized at $\tau=0$.
Define the terminal laws
\[
q_1^c := (\Phi_1^\theta(\cdot;\,c))_\#\nu_0(\cdot\mid c)=K_\theta(c,\cdot),
\qquad
p_1^c := (\Phi_1^\star(\cdot;\,c))_\#\nu_0(\cdot\mid c).
\]
Then
\begin{equation}\label{eq:K:w2-bound}
\Wtwo\bigl(p_1^c,q_1^c\bigr)
\le
e^{L_c}\,\sqrt{\mathcal L_{\mathrm{FM}}(c)}.
\end{equation}
In particular, if $p_1^c=K^\star(c,\cdot)$ (as intended by Assumption~\ref{ass:K:fm-reference}), then
$\Wtwo\bigl(K^\star(c,\cdot),K_\theta(c,\cdot)\bigr)\le e^{L_c}\sqrt{\mathcal L_{\mathrm{FM}}(c)}$.
\end{theorem}

\begin{proof}
Let $X_0\sim \nu_0(\cdot\mid c)$ and define the coupled processes
$X_\tau^\star := \Phi_\tau^\star(X_0;\,c)$ and $X_\tau^\theta := \Phi_\tau^\theta(X_0;\,c)$ on the same probability space.
Then $\Law(X_1^\star)=p_1^c$ and $\Law(X_1^\theta)=q_1^c$, so by definition of $\Wtwo$,
\begin{equation}\label{eq:K:w2-coupling}
\Wtwo(p_1^c,q_1^c)^2 \le \E\bigl\|X_1^\theta - X_1^\star\bigr\|^2.
\end{equation}
Set $\Delta_\tau := X_\tau^\theta - X_\tau^\star$. Then
\[
\dot\Delta_\tau
= v_\theta(\tau,X_\tau^\theta;\,c)-v^\star(\tau,X_\tau^\star;\,c)
= \bigl[v_\theta(\tau,X_\tau^\theta;\,c)-v_\theta(\tau,X_\tau^\star;\,c)\bigr]
  + \bigl[v_\theta(\tau,X_\tau^\star;\,c)-v^\star(\tau,X_\tau^\star;\,c)\bigr].
\]
By Lipschitz continuity of $v_\theta$,
\[
\|\dot\Delta_\tau\|
\le L_c\|\Delta_\tau\| + \bigl\|v_\theta(\tau,X_\tau^\star;\,c)-v^\star(\tau,X_\tau^\star;\,c)\bigr\|.
\]
Gronwall's inequality yields
\[
\|\Delta_1\|
\le e^{L_c}\int_0^1 \bigl\|v_\theta(\tau,X_\tau^\star;\,c)-v^\star(\tau,X_\tau^\star;\,c)\bigr\|\,d\tau.
\]
Apply Cauchy--Schwarz in $\tau$ and then take expectations:
\[
\E\|\Delta_1\|^2
\le e^{2L_c}\int_0^1 \E\bigl\|v_\theta(\tau,X_\tau^\star;\,c)-v^\star(\tau,X_\tau^\star;\,c)\bigr\|^2\,d\tau.
\]
Finally, $\Law(X_\tau^\star)=\pi_\tau^c$ by construction of the reference flow (Assumption~\ref{ass:K:fm-reference}),
so the right-hand side equals $e^{2L_c}\mathcal L_{\mathrm{FM}}(c)$.
Combine with \eqref{eq:K:w2-coupling} to obtain \eqref{eq:K:w2-bound}.
\end{proof}

\begin{remark}[What is (and is not) being assumed]
The key structural requirement is that the FM population target $v^\star$ transports $\nu_0(\cdot\mid c)$ to the true conditional law,
and that the FM risk is measured along the corresponding path marginals $\pi_\tau^c$.
This is precisely the regime in which flow matching is an operator-learning method: it learns a drift whose induced pushforward matches the target law.
The bound \eqref{eq:K:w2-bound} is intentionally ``first-principles'' and does not rely on KL/TV machinery.
\end{remark}

\subsubsection{From $\Wtwo$ to MMD (and hence to Appendix-K mixing bounds)}
\label{app:K:w2-to-mmd}

We now convert $\Wtwo$ control into MMD control for common choices of kernels.
The only ingredient is that RKHS unit-ball functions are Lipschitz (with a kernel-dependent constant).

\begin{lemma}[RKHS unit ball is uniformly Lipschitz under a gradient bound]
\label{lem:K:rkhs-lipschitz}
Let $k:\U\times\U\to\R$ be a $C^1$ positive definite kernel on $\U\subseteq\R^d$ with RKHS $\mathcal H_k$.
Assume there exists $L_k<\infty$ such that
\begin{equation}\label{eq:K:kernel-grad-bound}
\sup_{x\in\U}\ \bigl\|\nabla_x k(x,\cdot)\bigr\|_{\mathcal H_k} \le L_k.
\end{equation}
Then every $f\in\mathcal H_k$ satisfies $\mathrm{Lip}(f)\le L_k\|f\|_{\mathcal H_k}$.
Consequently,
\begin{equation}\label{eq:K:mmd-w1}
\MMD_k(\mu,\nu) \le L_k\,\Wone(\mu,\nu)\le L_k\,\Wtwo(\mu,\nu).
\end{equation}
\end{lemma}

\begin{proof}
For $f\in\mathcal H_k$ and $x\in\U$, the reproducing property gives
\[
\partial_{x_i} f(x) = \bigl\langle f,\ \partial_{x_i}k(x,\cdot)\bigr\rangle_{\mathcal H_k},
\]
hence $|\partial_{x_i} f(x)|\le \|f\|_{\mathcal H_k}\,\|\partial_{x_i}k(x,\cdot)\|_{\mathcal H_k}$.
Summing over coordinates and using \eqref{eq:K:kernel-grad-bound} yields $\|\nabla f(x)\|\le L_k\|f\|_{\mathcal H_k}$, uniformly in $x$,
so $\mathrm{Lip}(f)\le L_k\|f\|_{\mathcal H_k}$.

For $\|f\|_{\mathcal H_k}\le 1$, $f$ is $L_k$-Lipschitz, hence by the Kantorovich--Rubinstein duality,
\[
\left|\int f\,d\mu-\int f\,d\nu\right|\le L_k\,\Wone(\mu,\nu).
\]
Taking the supremum over $\|f\|_{\mathcal H_k}\le 1$ gives $\MMD_k(\mu,\nu)\le L_k\Wone(\mu,\nu)$, and $\Wone\le\Wtwo$ yields \eqref{eq:K:mmd-w1}.
\end{proof}

\begin{corollary}[FM risk controls conditional MMD]
\label{cor:K:fm-to-mmd}
Under the assumptions of Theorem~\ref{thm:K:fm-to-w2} and Lemma~\ref{lem:K:rkhs-lipschitz},
\begin{equation}\label{eq:K:fm-to-mmd}
\MMD_k\bigl(K^\star(c,\cdot),K_\theta(c,\cdot)\bigr)
\le L_k\,e^{L_c}\,\sqrt{\mathcal L_{\mathrm{FM}}(c)}.
\end{equation}
\end{corollary}

\begin{proof}
Combine Theorem~\ref{thm:K:fm-to-w2} with Lemma~\ref{lem:K:rkhs-lipschitz}.
\end{proof}

\paragraph{From conditional bounds to $\varepsilon_{\mathcal G}$.}
Define the worst-case conditional FM risk and Lipschitz constant over a relevant set of contexts $\Ccal_{\mathrm{test}}\subseteq\Ccal$:
\[
\mathcal L_{\mathrm{FM}}^{\sup} := \sup_{c\in\Ccal_{\mathrm{test}}}\mathcal L_{\mathrm{FM}}(c),
\qquad
L^{\sup} := \sup_{c\in\Ccal_{\mathrm{test}}} L_c.
\]
Then \eqref{eq:K:fm-to-mmd} implies
\begin{equation}\label{eq:K:eps-from-fm}
\varepsilon_{\MMD_k}
:= \sup_{c\in\Ccal_{\mathrm{test}}} \MMD_k\bigl(K^\star(c,\cdot),K_\theta(c,\cdot)\bigr)
\le L_k\,e^{L^{\sup}}\,\sqrt{\mathcal L_{\mathrm{FM}}^{\sup}}.
\end{equation}

\begin{remark}[Supremum vs.\ average and what the training loss actually controls]
The population training loss typically controls an \emph{average} of $\mathcal L_{\mathrm{FM}}(c)$ over contexts $c$
drawn from the training distribution (and in sequence training, from the rollout-induced distribution under teacher forcing).
To use \eqref{eq:K:eps-from-fm} with a supremum, one may either:
(i) restrict $\Ccal_{\mathrm{test}}$ to a compact/high-probability set and invoke uniform generalization,
or (ii) replace $\varepsilon_{\MMD_k}$ by an averaged kernel discrepancy, which is already sufficient for many operator perturbation bounds
(via Lemma~\ref{lem:K:operator-discrepancy} without the final $\le\varepsilon_{\mathcal G}$ step).
This appendix keeps the supremum form for clarity, but the same mechanism works with averages.
\end{remark}

\subsection{Putting it together: FM yields invariant-measure control (in MMD)}
\label{app:K:put-together}

Combining the bounds yields a direct implication from the FM training target to the invariant-measure discrepancy.

\begin{corollary}[FM $\Rightarrow$ invariant-measure bound in MMD under mixing]
\label{cor:K:fm-to-invariant}
Assume:
(i) the mixing/contractivity hypothesis of Assumption~\ref{ass:K:contractive} holds for $D_{\mathcal G}=\MMD_k$ with constant $\rho\in[0,1)$;
(ii) the FM-to-MMD bound \eqref{eq:K:eps-from-fm} holds on $\Ccal_{\mathrm{test}}$;
and (iii) $P$ and $P_\theta$ admit invariant measures $\mu^\star$ and $\mu_\theta^\star$.
Then
\[
\MMD_k(\mu^\star,\mu_\theta^\star)
\le \frac{1}{1-\rho}\,\varepsilon_{\MMD_k}
\le \frac{L_k\,e^{L^{\sup}}}{1-\rho}\,\sqrt{\mathcal L_{\mathrm{FM}}^{\sup}}.
\]
\end{corollary}

\begin{proof}
Apply Corollary~\ref{cor:K:invariant} with $D_{\mathcal G}=\MMD_k$ and substitute \eqref{eq:K:eps-from-fm}.
\end{proof}

\begin{remark}[Relation to MMD-regularized training baselines]
MMD-regularized baselines typically attempt to directly reduce $\MMD_k(\mu^\star,\mu_\theta^\star)$ by matching long-run statistics.
Corollary~\ref{cor:K:fm-to-invariant} formalizes that, in a mixing regime, learning the one-step kernel via flow matching
already upper bounds the stationary MMD error (up to the amplification factor $(1-\rho)^{-1}$), and does so in the same topology.
This does \emph{not} claim that the contractivity constant $\rho$ is easy to estimate for high-dimensional turbulent dynamics;
rather, it states the structural implication: operator (kernel) accuracy is a stronger control knob than fixed-point matching.
\end{remark}

\end{document}